\documentclass{article}

\usepackage[nonatbib,preprint]{neurips_2019}
\usepackage[numbers]{natbib}

\usepackage{graphicx}
\usepackage[utf8]{inputenc} %
\usepackage[T1]{fontenc}    %
\usepackage{hyperref}       %
\usepackage{url}            %
\usepackage{booktabs}       %
\usepackage{amsfonts}       %
\usepackage{nicefrac}       %
\usepackage{microtype}      %
\usepackage{amsmath}
\usepackage{color}
\usepackage{algorithm}
\usepackage{algorithmic}
\usepackage{subcaption}
\usepackage{wrapfig}
\usepackage{multirow}
\usepackage{tikz}
\usepackage{amsthm}
\usepackage{amssymb}
\usetikzlibrary{positioning}
\usepackage{mathtools}

\newtheorem{theorem}{Theorem}[section]
\newtheorem{assumption}{Assumption}[section]

\newtheorem{lemma}{Lemma}[theorem]

\newcommand{\policy}{\pi}

\newcommand{\transitions}{T}

\newcommand{\behavior}{{\pi_\beta}}
\newcommand{\bellman}{\mathcal{B}}

\newcommand{\Qfeat}{\mathbf{F}}
\newcommand{\hatbehavior}{\hat{\pi}_\beta}

\newcommand{\bx}{\mathbf{x}}

\newcommand{\bs}{\mathbf{s}}
\newcommand{\ba}{\mathbf{a}}
\newcommand{\bM}{\mathbf{M}}

\newcommand{\E}{\mathbb{E}}

\DeclareMathOperator\supp{supp}

\title{Conservative Q-Learning\\ for Offline Reinforcement Learning}

\author{%
  Aviral Kumar$^1$, Aurick Zhou$^1$, George Tucker$^2$, Sergey Levine$^{1,2}$ \\
  $^1$UC Berkeley, $^2$Google Research, Brain Team\\
  \texttt{aviralk@berkeley.edu}
}

\begin{document}

\maketitle

\begin{abstract}
Effectively leveraging large, previously collected datasets in reinforcement learning (RL) is a key challenge for large-scale real-world applications. Offline RL algorithms promise to learn effective policies from previously-collected, static datasets without further interaction. However, in practice, offline RL presents a major challenge, and standard off-policy RL methods can fail due to overestimation of values induced by the distributional shift between the dataset and the learned policy, especially when training on complex and multi-modal data distributions. In this paper, we propose \emph{conservative Q-learning (CQL)}, which aims to address these limitations by learning a conservative Q-function such that the expected value of a policy under this Q-function lower-bounds its true value. 
We theoretically show that CQL produces a lower bound on the value of the current policy and that it can be incorporated into a policy learning procedure with theoretical improvement guarantees. In practice, CQL augments the standard Bellman error objective with a simple Q-value regularizer which is straightforward to implement on top of existing deep Q-learning and actor-critic implementations. On both discrete and continuous control domains, we show that CQL substantially outperforms existing offline RL methods, often learning policies that attain 2-5 times higher final return, especially when learning from complex and multi-modal data distributions.
\end{abstract}

\section{Introduction}
\vspace{-4pt}
Recent advances in reinforcement learning (RL), especially when combined with expressive deep network function approximators, have produced promising results in domains ranging from robotics~\citep{kalashnikov2018qtopt} to strategy games~\citep{alphastar} and recommendation systems~\citep{li2010contextual}. However, applying RL to real-world problems consistently poses practical challenges: in contrast to the kinds of data-driven methods that have been successful in supervised learning~\citep{resnet,bert}, RL is classically regarded as an active learning process, where each training run requires active interaction with the environment. Interaction with the real world can be costly and dangerous, and the quantities of data that can be gathered online are substantially lower than the offline datasets that are used in supervised learning~\citep{imagenet}, which only need to be collected once. Offline RL, also known as batch RL, offers an appealing alternative~\citep{ernst2005tree,fujimoto2018off,kumar2019stabilizing,agarwal2019optimistic,jaques2019way,siegel2020keep,levine2020offline}. Offline RL algorithms learn from large, previously collected datasets, without interaction. This in principle can make it possible to leverage large datasets, but in practice fully offline RL methods pose major technical difficulties, stemming from the distributional shift between the policy that collected the data and the learned policy. This has made current results fall short of the full promise of such methods.

Directly utilizing existing value-based off-policy RL algorithms in an offline setting generally results in poor performance, due to issues with bootstrapping from out-of-distribution actions~\citep{kumar2019stabilizing,fujimoto2018off} and overfitting~\citep{fu2019diagnosing,kumar2019stabilizing,agarwal2019optimistic}. This typically manifests as erroneously optimistic value function estimates.
If we can instead learn a \emph{conservative} estimate of the value function, which provides a lower bound on the true values, this overestimation problem could be addressed. In fact, because policy evaluation and improvement typically only use the value of the policy, we can learn a less conservative lower bound Q-function, such that only the expected value of Q-function under the policy is lower-bounded, as opposed to a point-wise lower bound.
We propose a novel method for learning such conservative Q-functions via a simple modification to standard value-based RL algorithms. The key idea behind our method is to minimize values under an appropriately chosen distribution over state-action tuples, and then further tighten this bound by also incorporating a \emph{maximization} term over the data distribution.

Our primary contribution is an algorithmic framework, which we call conservative Q-learning (CQL), for learning conservative, lower-bound estimates of the value function, by regularizing the Q-values during training.
Our theoretical analysis of CQL shows that \textit{only} the expected value of this Q-function under the policy lower-bounds the true policy value, preventing extra under-estimation that can arise with point-wise lower-bounded Q-functions, that have typically been explored in the opposite context in exploration literature~\citep{osband2017posterior,jaksch2010near}.
We also empirically demonstrate the robustness of our approach to Q-function estimation error.
Our practical algorithm uses these conservative estimates for policy evaluation and offline RL. CQL can be implemented with less than \textbf{20} lines of code on top of a number of standard, online RL algorithms~\citep{haarnoja,dabney2018distributional}, simply by adding the CQL regularization terms to the Q-function update. In our experiments, we demonstrate the efficacy of CQL for offline RL, in domains with complex dataset compositions, where prior methods are typically known to perform poorly~\citep{d4rl} and domains with high-dimensional visual inputs~\citep{bellemare2013arcade,agarwal2019optimistic}. CQL outperforms prior methods by as much as \textbf{2-5x} on many benchmark tasks, and is the only method that can outperform simple behavioral cloning on a number of realistic datasets collected from human interaction.

\vspace{-10pt}
\section{Preliminaries}
\label{sec:background}
\vspace{-8pt}
The goal in reinforcement learning is to learn a policy that maximizes the expected cumulative discounted reward in a Markov decision process (MDP), which is defined by a tuple $(\mathcal{S}, \mathcal{A}, \transitions, r, \gamma)$.
$\mathcal{S}, \mathcal{A}$ represent state and action spaces, $\transitions(\bs' | \bs, \ba)$ and $r(\bs,\ba)$ represent the dynamics and reward function, and $\gamma \in (0,1)$ represents the discount factor. $\behavior(\ba|\bs)$ represents the behavior policy, $\mathcal{D}$ is the dataset, and $d^\behavior(\bs)$ is the discounted marginal state-distribution of $\behavior(\ba|\bs)$. The dataset $\mathcal{D}$ is sampled from $d^\behavior(\bs) \behavior(\ba|\bs)$. {On all states $\bs \in \mathcal{D}$, let $\hatbehavior(\ba|\bs_) := \frac{\sum_{\bs,\ba \in \mathcal{D}} \mathbf{1} [\bs = \bs , \ba = \ba]}{\sum_{\bs \in \mathcal{D}} \mathbf{1}[\bs = \bs]}$ denote the empirical behavior policy, at that state.} We assume that the rewards $r$ satisfy: $|r(\bs, \ba)| \leq R_{\max}$.

Off-policy RL algorithms based on dynamic programming maintain a parametric Q-function $Q_\theta(s, a)$ and, optionally, a parametric policy, $\pi_\phi(a|s)$. 
Q-learning methods train the Q-function by iteratively applying the Bellman optimality operator \mbox{$\mathcal{B}^*Q(\bs, \ba) = r(\bs, \ba) + \gamma \E_{\bs' \sim P(\bs'|\bs, \ba)}[\max_{\ba'} Q(\bs', \ba')]$}, and use exact or an approximate maximization scheme, such as CEM~\citep{kalashnikov2018qtopt} to recover the greedy policy. In an actor-critic algorithm, a separate policy is trained to maximize the Q-value.
Actor-critic methods alternate between computing $Q^\policy$ via (partial) policy evaluation,
by iterating the Bellman operator, $\mathcal{B}^\pi Q= r + \gamma P^\pi Q$, where $P^\pi$ is the transition matrix coupled with the policy: $P^\pi Q(\bs, \ba) = \E_{\bs' \sim \transitions(\bs' | \bs, \ba), \ba' \sim \pi(\ba'|\bs')} \left[ Q(\bs', \ba') \right],$
and improving the policy
$\policy(\ba|\bs)$ by updating it towards actions that maximize the expected Q-value. Since $\mathcal{D}$ typically does not contain all possible transitions $(\bs, \ba, \bs')$, the policy evaluation step actually uses an empirical Bellman operator that only backs up a single sample. We denote this operator $\hat{\bellman}^\policy$. Given a dataset $\mathcal{D} = \{(\bs, \ba, r \bs')\}$ of tuples from trajectories collected under a behavior policy $\behavior$:
\begin{small}
\begin{align*}
    \hat{Q}^{k+1} \leftarrow& \arg\min_{Q} \E_{\bs, \ba, \bs' \sim \mathcal{D}}\left[ \left((r(\bs, \ba) + \gamma \E_{\ba' \sim \hat{\policy}^k(\ba'|\bs')}[\hat{Q}^{k}(\bs', \ba')]) - Q(\bs, \ba)\right)^2 \right]~\text{(policy evaluation)}\\ 
    \hat{\policy}^{k+1} \leftarrow& \arg\max_{\policy} \E_{\bs \sim \mathcal{D}, \ba \sim \policy^k(\ba|\bs)}\left[\hat{Q}^{k+1}(\bs, \ba)\right]~~~ \text{(policy improvement)} 
\end{align*}
\end{small}
Offline RL algorithms based on this basic recipe suffer from action distribution shift~\citep{kumar2019stabilizing,wu2019behavior,jaques2019way,levine2020offline} during training, because the target values for Bellman backups in policy evaluation use actions sampled from the learned policy, $\policy^k$, but the Q-function is trained only on actions sampled from the behavior policy that produced the dataset $\mathcal{D}$, $\behavior$. Since $\policy$ is trained to maximize Q-values, it may be biased towards out-of-distribution (OOD) actions with erroneously high Q-values. In standard RL, such errors can be corrected by attempting an action in the environment and observing its actual value. However, the inability to interact with the environment makes it challenging to deal with Q-values for OOD actions in offline RL. Typical offline RL methods~\citep{kumar2019stabilizing,jaques2019way,wu2019behavior,siegel2020keep} mitigate this problem by constraining the learned policy~\citep{levine2020offline} away from OOD actions. {Note that Q-function training in offline RL does not suffer from state distribution shift, as the Bellman backup never queries the Q-function on out-of-distribution states. However, the policy may suffer from state distribution shift at test time.}

\vspace{-5pt}
\section{The Conservative Q-Learning (CQL) Framework}
\vspace{-5pt}
In this section, we develop a conservative Q-learning (CQL) algorithm, such that the expected value of a policy under the learned Q-function lower-bounds its true value. A lower bound on the Q-value prevents the over-estimation that is common in offline RL settings due to OOD actions and function approximation error~\citep{levine2020offline,kumar2019stabilizing}. {We use the term CQL to refer broadly to both Q-learning methods and actor-critic methods, though the latter also use an explicit policy.}
We start by focusing on the policy evaluation step in CQL, which can be used by itself as an off-policy evaluation procedure, or integrated into a complete offline RL algorithm, as we will discuss in Section~\ref{sec:framework}.

\vspace{-5pt}
\subsection{Conservative Off-Policy Evaluation}
\label{sec:policy_eval}
\vspace{-5pt}
We aim to estimate the value $V^\policy(\bs)$ of a target policy $\policy$ given access to a dataset, $\mathcal{D}$, generated by following a behavior policy $\behavior(\ba|\bs)$.
Because we are interested in preventing overestimation of the policy value, we learn a \textit{conservative}, lower-bound Q-function by additionally minimizing Q-values alongside a standard Bellman error objective.
Our choice of penalty is to minimize the expected Q-value under a particular distribution of state-action pairs, $\mu(\bs, \ba)$. 
Since standard Q-function training does not query the Q-function value at unobserved states, but queries the Q-function at unseen actions, 
we restrict $\mu$ to match the state-marginal in the dataset,
such that $\mu(\bs, \ba) = d^\behavior(\bs) \mu(\ba|\bs)$.
This gives rise to the iterative update for training the Q-function, as a function of a tradeoff factor $\alpha$:
\begin{equation}
    \small{\hat{Q}^{k+1} \leftarrow \arg\min_{Q}~ \alpha~ \E_{\bs \sim \mathcal{D}, \ba \sim \mu(\ba|\bs)}\left[Q(\bs, \ba)\right] + \frac{1}{2}~ \E_{\bs, \ba \sim \mathcal{D}}\left[\left(Q(\bs, \ba) - \hat{\bellman}^\policy \hat{Q}^{k} (\bs, \ba) \right)^2 \right],} 
    \label{eqn:objective_1}
\end{equation}
In Theorem~\ref{thm:min_q_underestimates}, we show that the resulting Q-function, $\hat{Q}^\policy \coloneqq \lim_{k \rightarrow \infty} \hat{Q}^k$, lower-bounds $Q^\policy$ at all $(\bs,\ba)$. 
However, we can substantially tighten this bound
if we are \textit{only} interested in estimating  $V^\policy(\bs)$. If we only require that the expected value of the $\hat{Q}^\pi$ under $\policy(\ba|\bs)$ lower-bound $V^\pi$, we can improve the bound by introducing an additional Q-value \textit{maximization} term under the data distribution, $\behavior(\ba|\bs)$, resulting in the iterative update (changes from Equation~\ref{eqn:objective_1} in red):
\begin{multline}
    \small{\hat{Q}^{k+1} \leftarrow \arg\min_{Q}~~ \alpha \cdot \left(\E_{\bs \sim \mathcal{D}, \ba \sim \mu(\ba|\bs)}\left[Q(\bs, \ba)\right] - \textcolor{red}{\E_{\bs \sim \mathcal{D}, \ba \sim \hatbehavior(\ba|\bs)}\left[Q(\bs, \ba)\right]} \right)} \\
    \small{+ \frac{1}{2}~ \E_{\bs, \ba, \bs' \sim \mathcal{D}}\left[\left(Q(\bs, \ba) - \hat{\bellman}^\policy \hat{Q}^{k} (\bs, \ba) \right)^2 \right]}.
    \label{eqn:modified_policy_eval}
\end{multline}
In Theorem~\ref{thm:cql_underestimates}, we show that, while the resulting Q-value 
$\hat{Q}^{\policy}$ may not be a point-wise lower-bound, we have $\E_{\policy(\ba|\bs)}[\hat{Q}^\pi(\bs, \ba)] \leq V^\pi(\bs)$
 when $\mu(\ba|\bs) = \policy(\ba|\bs)$. 
Intuitively, since Equation~\ref{eqn:modified_policy_eval} maximizes Q-values under the behavior policy $\hatbehavior$, Q-values for actions that are likely under $\hatbehavior$ might be overestimated, and hence $\hat{Q}^\policy$ may not lower-bound $Q^\policy$ pointwise. 
While in principle the maximization term can utilize other distributions besides $\hatbehavior(\ba|\bs)$, we prove in Appendix~\ref{app:maximizing_distributions} that the resulting value is not guaranteed to be a lower bound for other distribution besides $\hatbehavior(\ba|\bs)$.

\textbf{Theoretical analysis.}
We first note that Equations~\ref{eqn:objective_1} and \ref{eqn:modified_policy_eval} use the empirical Bellman operator, $\hat{\bellman}^\policy$, instead of the actual Bellman operator, $\bellman^\policy$.  Following~\citep{osband2016deep,jaksch2010near,o2018variational}, we use concentration properties of $\hat{\bellman}^\policy$ to control this error. Formally,
for all $\bs, \ba \in \mathcal{D}$, with  probability $\geq 1 - \delta$, $|\hat{\bellman}^\policy - \bellman^\policy|(\bs, \ba) \leq \frac{C_{r,T, \delta}}{\sqrt{|\mathcal{D}(\bs, \ba)|}}$, where $C_{r, T, \delta}$ is a constant dependent on the concentration properties (variance) of $r(\bs, \ba)$ and $\transitions(\bs'|\bs, \ba)$, and $\delta\in (0, 1)$ (see Appendix~\ref{app:handling_unobserved_actions} for more details). 
{For simplicity in the derivation, we assume that $\hatbehavior(\ba|\bs) > 0, \forall \ba \in \mathcal{A}~, \forall \bs \in \mathcal{D}$. $\frac{1}{\sqrt{|\mathcal{D}|}}$ denotes a vector of size $|\mathcal{S}| |\mathcal{A}|$ containing square root inverse counts for each state-action pair, except when $\mathcal{D}(\bs, \ba) = 0$, in which case the corresponding entry is a very large but finite value $\delta \geq \frac{2 R_{\max}}{1 - \gamma}$.}
Now, we show that the conservative Q-function learned by iterating Equation~\ref{eqn:objective_1} lower-bounds the true Q-function. Proofs can be found in Appendix~\ref{app:missing_proofs}.

\begin{theorem} %
\label{thm:min_q_underestimates}
For any $\mu(\ba|\bs)$ with $\supp \mu \subset \supp \hatbehavior$, with probability $\geq 1 - \delta$, $\hat{Q}^\pi$ (the Q-function obtained by iterating Equation~\ref{eqn:objective_1}) satisifies:
\begin{equation*}
\forall \bs \in \mathcal{D}, \ba, ~~ \hat{Q}^\policy(s, a) \leq Q^\policy(\bs, \ba) - \alpha \left[\left(I - \gamma P^\pi \right)^{-1} \frac{\mu}{\hatbehavior} \right](\bs, \ba) + \left[ (I - \gamma P^\pi)^{-1}\frac{C_{r, T, \delta} R_{\max}}{(1- \gamma) \sqrt{|\mathcal{D}|}} \right](\bs, \ba).    
\end{equation*}
Thus, if {$\alpha$ is sufficiently large},
then $\hat{Q}^\policy(\bs, \ba)  \leq Q^\policy(\bs, \ba), \forall \bs \in \mathcal{D}, \ba$. When $\hat{\bellman}^\policy = \bellman^\policy$, any $\alpha > 0$ guarantees $\hat{Q}^\policy(\bs, \ba)  \leq Q^\policy(\bs, \ba), \forall \bs \in \mathcal{D}, \ba \in \mathcal{A}$.
\end{theorem}
Next, we show that Equation~\ref{eqn:modified_policy_eval} lower-bounds the expected value under the policy $\policy$, when $\mu = \policy$. We also show that Equation~\ref{eqn:modified_policy_eval} does not lower-bound the Q-value estimates pointwise. {For this result, we abuse notation and assume that $\frac{1}{\sqrt{|\mathcal{D}|}}$ refers to a vector of inverse square root of only state counts, with a similar correction as before used to handle the entries of this vector at states with zero counts.}
\begin{theorem}[Equation~\ref{eqn:modified_policy_eval} results in a tighter lower bound]
\label{thm:cql_underestimates}
The value of the policy under the Q-function from Equation~\ref{eqn:modified_policy_eval}, $\hat{V}^\policy(\bs) = \E_{\policy(\ba|\bs)}[\hat{Q}^\policy(\bs, \ba)]$, lower-bounds the true value of the policy obtained via exact policy evaluation, $V^\policy(\bs) = \E_{\policy(\ba|\bs)}[Q^\policy(\bs, \ba)]$, when $\mu = \policy$, according to:
\begin{equation*}
\forall \bs \in \mathcal{D}, ~~ \hat{V}^\policy(\bs) \leq V^\policy(\bs) - \alpha \left[\left(I - \gamma P^\pi \right)^{-1} \E_{\policy}\left[\frac{\policy}{\hatbehavior} - 1 \right] \right](\bs) + \left[ (I - \gamma P^\pi)^{-1} \frac{C_{r, T, \delta} R_{\max}}{(1- \gamma) \sqrt{|\mathcal{D}|}}\right](\bs).
\end{equation*}
Thus, if $\alpha > \frac{C_{r, T} R_{\max}}{1 - \gamma} \cdot \max_{\bs \in \mathcal{D}} \frac{1}{|\sqrt{|\mathcal{D}(\bs)|}} \cdot \left[\sum_{\ba} \policy(\ba|\bs) (\frac{\policy(\ba|\bs)}{\hatbehavior(\ba|\bs))} - 1)\right]^{-1}$
,~ $\forall \bs \in \mathcal{D},~\hat{V}^\policy(\bs) \leq {V}^\policy(\bs)$, with probability $\geq 1 - \delta$. When $\hat{\bellman}^\policy = {\bellman}^\policy$, then any $\alpha > 0$ guarantees $\hat{V}^\policy(\bs) \leq V^\policy(\bs), \forall \bs \in \mathcal{D}$.
\end{theorem}
The analysis presented above assumes that no function approximation is used in the Q-function, meaning that each iterate can be represented exactly.
We can further generalize the result in Theorem~\ref{thm:cql_underestimates} to the case of both linear function approximators and non-linear neural network function approximators, where the latter builds on the neural tangent kernel (NTK) framework~\citep{ntk}. Due to space constraints, we present these results in Theorem~\ref{thm:policy_eval_func_approx} and Theorem~\ref{corr:nonlinear_ntk} in Appendix~\ref{app:cql_linear_non_linear}. 

\textbf{In summary}, we showed that the basic CQL evaluation in Equation~\ref{eqn:objective_1} learns a Q-function that lower-bounds the true Q-function $Q^\pi$, and the evaluation in Equation~\ref{eqn:modified_policy_eval} provides a \emph{tighter} lower bound on the expected Q-value of the policy $\pi$. For suitable $\alpha$, both bounds hold under sampling error and function approximation. {We also note that as more data becomes available and $|\mathcal{D}(\bs, \ba)|$ increases, the theoretical value of $\alpha$ that is needed to guarantee a lower bound decreases, which indicates that in the limit of infinite data, a lower bound can be obtained by using extremely small values of $\alpha$.} Next, we will extend on this result into a complete RL algorithm.

\vspace{-5pt}
\subsection{Conservative Q-Learning for Offline RL}
\label{sec:framework}
\vspace{-5pt}
We now present a general approach for offline policy learning, which we refer to as conservative Q-learning (CQL). 
As discussed in Section~\ref{sec:policy_eval}, we can obtain Q-values that lower-bound the value of a policy $\policy$
by solving Equation~\ref{eqn:modified_policy_eval} with $\mu = \policy$. How should we utilize this for policy optimization? We could alternate between performing full off-policy evaluation for each policy iterate, $\hat{\policy}^k$, and one step of policy improvement. However, this can be computationally expensive. Alternatively, since the policy $\hat{\policy}^k$ is typically derived from the Q-function, we could instead choose $\mu(\ba|\bs)$ to approximate the policy that would maximize the current Q-function iterate, 
thus giving rise to an online algorithm.

We can formally capture such online algorithms by defining a family of optimization problems over $\mu(\ba|\bs)$, presented below, with modifications from Equation~\ref{eqn:modified_policy_eval} marked in red. An instance of this family is denoted by CQL($\mathcal{R}$) and is characterized by a particular choice of regularizer $\mathcal{R}(\mu)$:
\begin{multline}
    \label{eqn:cql_framework}
    \small{\min_{Q} \textcolor{red}{\max_{\mu}}~~ \alpha \left(\E_{\bs \sim \mathcal{D}, \ba \sim \textcolor{red}{\mu(\ba|\bs)}}\left[Q(\bs, \ba)\right] - \E_{\bs \sim \mathcal{D}, \ba \sim \hatbehavior(\ba|\bs)}\left[Q(\bs, \ba)\right] \right)}\\
    \small{+ \frac{1}{2}~ \E_{\bs, \ba, \bs' \sim \mathcal{D}}\left[\left(Q(\bs, \ba) - \hat{\bellman}^{\policy_k} \hat{Q}^{k} (\bs, \ba) \right)^2 \right] + \textcolor{red}{\mathcal{R}(\mu)} ~~~ \left(\text{CQL}(\mathcal{R})\right).}
\end{multline}

\textbf{Variants of CQL.} To demonstrate the generality of the CQL family of optimization problems, we discuss two specific instances within this family that are of special interest, and we evaluate them empirically in Section~\ref{sec:experiments}. 
If we choose $\mathcal{R}(\mu)$ to be the KL-divergence against a prior distribution, $\rho(\ba|\bs)$, i.e., $\mathcal{R}(\mu) = -D_{\mathrm{KL}}(\mu, \rho)$, then we get $\mu(\ba|\bs) \propto \rho(\ba|\bs) \cdot \exp (Q(\bs, \ba))$ (for a derivation, see Appendix~\ref{app:cql_variants}). Frist, if $\rho = \text{Unif}(\ba)$,
then the first term in Equation~\ref{eqn:cql_framework} corresponds to a soft-maximum of the Q-values at any state $\bs$ and gives rise to the following variant of Equation~\ref{eqn:cql_framework}, called CQL($\mathcal{H}$):
\begin{equation}
    \small{\min_{Q}~ \alpha \E_{\bs \sim \mathcal{D}}\left[\log \sum_{\ba} \exp(Q(\bs, \ba))\!-\!\E_{\ba \sim \hatbehavior(\ba|\bs)}\left[Q(\bs, \ba)\right]\right]\!+\!\frac{1}{2} \E_{\bs, \ba, \bs' \sim \mathcal{D}}\left[\left(Q - \hat{\bellman}^{\policy_k} \hat{Q}^{k} \right)^2 \right]\!.}
    \label{eqn:practical_objective}
\end{equation}
Second, if $\rho(\ba|\bs)$ is chosen to be the previous policy $\hat{\policy}^{k-1}$, the first term in Equation~\ref{eqn:practical_objective} is replaced by an exponential weighted average of Q-values of actions from the chosen $\hat{\policy}^{k-1}(\ba|\bs)$. Empirically, we find that this variant can be more stable with high-dimensional action spaces (e.g., Table~\ref{table:adroit_antmaze}) where it is challenging to estimate $\log \sum_{\ba} \exp$ via sampling due to high variance. In Appendix~\ref{app:cql_variants}, we discuss an additional variant of CQL, drawing connections to distributionally robust optimization~\citep{namkoong2017variance}.  
We will discuss a practical instantiation of a CQL deep RL algorithm in Section~\ref{sec:practical_alg}. CQL can be instantiated as either a Q-learning algorithm (with $\bellman^*$ instead of $\bellman^{\policy}$ in Equations~\ref{eqn:cql_framework}, \ref{eqn:practical_objective}) or as an actor-critic algorithm. 

\textbf{Theoretical analysis of CQL.}
Next, we will theoretically analyze CQL to show that the policy updates derived in this way are indeed ``conservative'', in the sense that each successive policy iterate is optimized against a lower bound on its value. For clarity, we state the results in the absence of finite-sample error, in this section, but sampling error can be incorporated in the same way as Theorems~\ref{thm:min_q_underestimates} and \ref{thm:cql_underestimates}, and we discuss this in Appendix~\ref{app:missing_proofs}.   Theorem~\ref{thm:cql_underestimation} shows that any variant of the CQL family learns Q-value estimates that lower-bound the actual Q-function under the action-distribution defined by the policy, $\policy^{k}$, under mild regularity conditions (slow updates on the policy).
\begin{theorem}[CQL learns lower-bounded Q-values]
\label{thm:cql_underestimation}
Let $\policy_{\hat{Q}^k}(\ba | \bs) \propto \exp(\hat{Q}^k(\bs, \ba))$ and assume that $D_{\mathrm{TV}}(\hat{\policy}^{k+1}, \pi_{\hat{Q}^k}) \leq \varepsilon$ (i.e., $\hat{\policy}^{k+1}$ changes slowly w.r.t to $\hat{Q}^k$). Then, the policy value under $\hat{Q}^k$, lower-bounds the actual policy value, \mbox{$\hat{V}^{k+1}(\bs) \leq V^{k+1}(\bs) \forall \bs$} if
\begin{equation*}
\small{\E_{\pi_{\hat{Q}^k}(\ba | \bs)} \left[ \frac{\pi_{\hat{Q}^k}(\ba | \bs)}{\hatbehavior(\ba|\bs)} -1 \right] \geq \max_{\ba \text{~s.t.~} \hatbehavior(\ba|\bs) > 0} \left(\frac{\pi_{\hat{Q}^k}(\ba | \bs)}{\hatbehavior(\ba|\bs)} \right) \cdot \varepsilon}. 
\end{equation*}
\end{theorem}
The LHS of this inequality is equal to the amount of conservatism induced in the value, $\hat{V}^{k+1}$ in iteration $k+1$ of the CQL update,
if the learned policy were equal to soft-optimal policy for $\hat{Q}^k$, i.e., when $\hat{\policy}^{k+1} = \pi_{\hat{Q}^k}$. However, as the actual policy, $\hat{\policy}^{k+1}$, may be different, the RHS is the maximal amount of potential overestimation due to this difference. To get a lower bound, we require the amount of underestimation to be higher, which is obtained if $\varepsilon$ is small, i.e. the policy changes slowly.  

Our final result shows that CQL Q-function update is ``gap-expanding'', by which we mean that the difference in Q-values at in-distribution actions and over-optimistically erroneous out-of-distribution actions is higher than the corresponding difference under the actual Q-function. This implies that the policy $\policy^k(\ba|\bs) \propto \exp(\hat{Q}^k(\bs, \ba))$, is constrained to be closer to the dataset distribution, $\hatbehavior(\ba|\bs)$, thus the CQL update implicitly prevents the detrimental effects of OOD action and distribution shift, which has been a major concern in offline RL settings~\citep{kumar2019stabilizing,levine2020offline,fujimoto2018off}.
\begin{theorem}[CQL is gap-expanding] 
\label{thm:gap_amplify}
At any iteration $k$, CQL expands the difference in expected Q-values under the behavior policy $\behavior(\ba|\bs)$ and $\mu_k$, such that for large enough values of $\alpha_k$, we have that \mbox{$\forall \bs, ~\E_{\behavior(\ba|\bs)}[\hat{Q}^k(\bs, \ba)] - \E_{\mu_k(\ba|\bs)}[\hat{Q}^k(\bs, \ba)] > \E_{\behavior(\ba|\bs)}[{Q}^k(\bs, \ba)] - \E_{\mu_k(\ba|\bs)}[{Q}^k(\bs, \ba)]$}.
\end{theorem}
When function approximation or sampling error makes OOD actions have higher learned Q-values, CQL backups are expected to be more robust, in that the policy is updated using Q-values that prefer in-distribution actions. 
As we will empirically show in Appendix~\ref{app:gap_amplify}, prior offline RL methods that do not explicitly constrain or regularize the Q-function may not enjoy such robustness properties.

\textbf{To summarize}, we showed that the CQL RL algorithm learns lower-bound Q-values with large enough $\alpha$, meaning that the final policy attains \emph{at least} the estimated value. We also showed that the Q-function is \emph{gap-expanding}, meaning that it should only ever \emph{over-estimate} the gap between in-distribution and out-of-distribution actions, preventing OOD actions.

\subsection{Safe Policy Improvement Guarantees}
In Section~\ref{sec:policy_eval} we proposed novel objectives for Q-function training such that the expected value of a policy under the resulting Q-function lower bounds the actual performance of the policy. In Section~\ref{sec:framework}, we used the learned conservative Q-function for policy improvement. In this section, we show that this procedure actually optimizes a well-defined objective and provide a safe policy improvement result for CQL, along the lines of Theorems 1 and 2 in \citet{laroche2017safe}.

To begin with, we define \emph{empirical return} of any policy $\policy$, ${J}(\policy, \hat{M})$, which is equal to the discounted return of a policy $\policy$ in the \emph{empirical} MDP, $\hat{M}$, that is induced by the transitions observed in the dataset $\mathcal{D}$, i.e. $\hat{M} = \{s, a, r, s' \in \mathcal{D}\}$. $J(\policy, M)$ refers to the expected discounted return attained by a policy $\policy$ in the actual underlying MDP, $M$. In Theorem~\ref{thm:well_defined_obj}, we first show that CQL (Equation~\ref{eqn:modified_policy_eval}) optimizes a well-defined penalized RL empirical objective. All proofs are found in Appendix~\ref{app:safe_pi}. 

\begin{theorem}
\label{thm:well_defined_obj}
Let $\hat{Q}^\pi$ be the fixed point of Equation~\ref{eqn:modified_policy_eval}, then $\policy^*(\ba|\bs) := \arg\max_{\policy} \mathbb{E}_{\bs \sim \rho(\bs)}[\hat{V}^\pi(\bs)]$ is equivalently obtained by solving:
\begin{equation}
\label{eqn:policy_optimality_main}
    \small{\policy^*(\ba|\bs) \leftarrow \arg \max_{\policy}~~ J(\pi, \hat{M}) - \alpha \frac{1}{1 - \gamma} \mathbb{E}_{\bs \sim d^\policy_{\hat{M}}(\bs)}\left[D_{\text{CQL}}(\policy, \hatbehavior)(\bs) \right]},
\end{equation}
where $D_{\text{CQL}}(\policy, \behavior)(\bs) := \sum_{\ba} \policy(\ba|\bs) \cdot \left(\frac{\policy(\ba|\bs)}{\behavior(\ba|\bs)} - 1 \right)$.
\end{theorem}
Intuitively, Theorem~\ref{thm:well_defined_obj} says that CQL optimizes the return of a policy in the empirical MDP, $\hat{M}$, while also ensuring that the learned policy $\policy$ is not too different from the behavior policy, $\hatbehavior$ via a penalty that depends on $D_{\text{CQL}}$. Note that this penalty is implicitly introduced by virtue by the gap-expanding (Theorem~\ref{thm:gap_amplify}) behavior of CQL. Next, building upon Theorem~\ref{thm:well_defined_obj} and the analysis of CPO~\citep{achiam2017constrained}, we show that CQL provides a $\zeta$-safe policy improvement over $\hatbehavior$.   

\begin{theorem}
\label{thm:zeta_safe}
Let $\policy^*(\ba|\bs)$ be the policy obtained by optimizing Equation~\ref{eqn:policy_optimality_main}. Then, the policy $\policy^*(\ba|\bs)$ is a $\zeta$-safe policy improvement over $\hatbehavior$ in the actual MDP $M$, i.e., $J(\policy^*, M) \geq J(\hatbehavior, M) - \zeta$ with high probability $1 - \delta$, where $\zeta$ is given by,
\begin{multline}
    \label{eqn:performance_relation}
    \small{\zeta = 2\left({\frac{C_{r, \delta}}{1 - \gamma} + \frac{\gamma R_{\max} C_{T, \delta}}{(1 - \gamma)^2}}  \right) \mathbb{E}_{\bs \sim d^{\policy^*}_{\hat{M}}(\bs)}\left[ \frac{\sqrt{|\mathcal{A}|}}{\sqrt{|\mathcal{D}(\bs)|}} \sqrt{ D_{\text{CQL}}(\policy^*, \hatbehavior)(\bs) + 1} \right]} - \\ \small{\underbrace{\left( J(\policy^*, \hat{M}) - J(\hatbehavior, \hat{M})  \right)}_{\geq \alpha \frac{1}{1 - \gamma} \mathbb{E}_{\bs \sim d^{\policy^*}_{\hat{M}}(\bs)}\left[D_{\text{CQL}}(\policy^*, \hatbehavior)(\bs) \right]}.}
\end{multline}
\end{theorem}
The expression of $\zeta$ in Theorem~\ref{thm:zeta_safe} consists of two terms: the first term captures the decrease in policy performance in $M$, that occurs due to the mismatch between $\hat{M}$ and $M$, also referred to as \emph{sampling error}. The second term captures the increase in policy performance due to CQL in empirical MDP, $\hat{M}$. The policy $\policy^*$ obtained by optimizing $\policy$ against the CQL Q-function improves upon the behavior policy, $\hatbehavior$ for suitably chosen values of $\alpha$. When sampling error is small, i.e., $|\mathcal{D}(\bs)|$ is large, then smaller values of $\alpha$ are enough to provide an improvement over the behavior policy. 

\textbf{To summarize,} CQL optimizes a well-defined, penalized empirical RL objective, and performs high-confidence safe policy improvement over the behavior policy. The extent of improvement is negatively influenced by higher sampling error, which decays as more samples are observed.  

\vspace{-7pt}
\section{Practical Algorithm and Implementation Details}
\label{sec:practical_alg}
\vspace{-7pt}
\begin{wrapfigure}{r}{0.6\textwidth}
\begin{small}
\vspace{-24pt}
\begin{minipage}[t]{0.99\linewidth}
\begin{algorithm}[H]
\small
\caption{Conservative Q-Learning (both variants)}
\label{alg:practical_alg}
\begin{algorithmic}[1]
    \STATE Initialize Q-function, $Q_\theta$, and optionally a policy, $\pi_\phi$.
    \FOR{step $t$ in \{1, \dots, N\}}
        \STATE Train the Q-function using $G_Q$ gradient steps on objective from Equation~\ref{eqn:practical_objective} \\
        \mbox{$\theta_t := \theta_{t-1} - \eta_Q \nabla_\theta \textcolor{red}{\text{CQL}(\mathcal{R})(\theta)}$}\\
        (Use $\bellman^*$ for Q-learning, $\bellman^{\policy_{\phi_t}}$ for actor-critic)
        \STATE \underline{(only with actor-critic)} Improve policy $\pi_\phi$ via $G_\pi$ gradient steps on $\phi$ with SAC-style entropy regularization:\\
        \mbox{$\phi_{t} := \phi_{t-1} + \eta_\pi \mathbb{E}_{\bs \sim \mathcal{D}, \ba \sim \pi_\phi(\cdot|\bs)}[Q_\theta(\bs, \ba)\! -\! \log \pi_\phi(\ba|\bs)] $}
    \ENDFOR
\end{algorithmic}
\end{algorithm}
\end{minipage}
\vspace{-14pt}
\end{small}
\end{wrapfigure}
We now describe two practical offline deep reinforcement learning methods based on CQL: an actor-critic variant and a Q-learning variant. Pseudocode is shown in Algorithm~\ref{alg:practical_alg}, with differences from conventional actor-critic algorithms (e.g., SAC~\citep{haarnoja}) and deep Q-learning algorithms (e.g., DQN~\citep{mnih2013playing}) in red.
Our algorithm uses the CQL($\mathcal{H}$) (or CQL($\mathcal{R}$) in general) objective from the CQL framework for training the Q-function $Q_\theta$, which is parameterized by a neural network with parameters $\theta$. For the actor-critic algorithm, a policy $\pi_\phi$ is trained as well. Our algorithm modifies the objective for the Q-function (swaps out Bellman error with CQL($\mathcal{H}$)) or CQL($\rho$)
in a standard actor-critic or Q-learning setting, as shown in Line 3. As discussed in Section~\ref{sec:framework}, due to the explicit penalty on the Q-function, CQL methods do not use a policy constraint,
unlike prior offline RL methods~\citep{kumar2019stabilizing,wu2019behavior,siegel2020keep,levine2020offline}.
Hence, we do not require fitting an additional behavior policy estimator, simplifying our method. %

\textbf{Implementation details.} Our algorithm requires an addition of only \textbf{20} lines of code on top of standard implementations of soft actor-critic (SAC)~\citep{haarnoja} for continuous control experiments and on top of QR-DQN~\citep{dabney2018distributional} for the discrete control experiments. The tradeoff factor, $\alpha$
is automatically tuned via Lagrangian dual gradient descent for continuous control, and is fixed at constant values described in Appendix~\ref{sec:experimental_details} for discrete control. We use default hyperparameters from SAC, except that the learning rate for the policy is chosen to be 3e-5 (vs 3e-4 or 1e-4 for the Q-function), as dictated by Theorem~\ref{thm:cql_underestimation}. Elaborate details are provided in Appendix~\ref{sec:experimental_details}.  

\vspace{-8pt}
\section{Related Work}
\label{sec:related}
\vspace{-8pt}
We now briefly discuss prior work in offline RL and off-policy evaluation, comparing and contrasting these works with our approach. More technical discussion of related work is provided in Appendix~\ref{sec:extended_related_work}.

\textbf{Off-policy evaluation (OPE).} Several different paradigms have been used to perform off-policy evaluation. Earlier works~\citep{precup2000eligibility,peshkin2002learning,precup2001off} used per-action importance sampling on Monte-Carlo returns to obtain an OPE return estimator. Recent approaches~\citep{liu2018breaking,gelada2019off,nachum2019dualdice,Zhang2020GenDICE:} use marginalized importance sampling by directly estimating the state-distribution importance ratios via some form of dynamic programming~\citep{levine2020offline} and typically exhibit less variance than per-action importance sampling at the cost of bias. Because these methods use dynamic programming, they can suffer from OOD actions
~\citep{levine2020offline,gelada2019off,hallak2017consistent,nachum2019dualdice}. In contrast, the regularizer in CQL explicitly addresses the impact of OOD actions due to its gap-expanding behavior, and obtains conservative value estimates.

\textbf{Offline RL.} As discussed in Section~\ref{sec:background}, offline Q-learning methods suffer from issues pertaining to OOD actions. Prior works have attempted to solve this problem by constraining the learned policy to be ``close'' to the behavior policy, for example as measured by  KL-divergence~\citep{jaques2019way,wu2019behavior,peng2019awr,siegel2020keep}, Wasserstein distance~\citep{wu2019behavior}, or MMD~\citep{kumar2019stabilizing}, and then 
only using actions sampled from this constrained policy in the Bellman backup or applying a value penalty. SPIBB~\citep{laroche2017safe,nadjahi2019safe} methods bootstrap using the behavior policy in a Q-learning algorithm for unseen actions. Most of these methods require a separately estimated model to the behavior policy, $\behavior(\ba|\bs)$~\citep{fujimoto2018off,kumar2019stabilizing,wu2019behavior,jaques2019way,siegel2020keep,simao2019safe}, and {are thus limited by their ability to accurately estimate the unknown behavior policy~\citep{nair2020accelerating}, which might be especially complex in settings where the data is collected from multiple sources~\citep{levine2020offline}. In contrast, CQL does not require estimating the behavior policy.} Prior work has explored some forms of Q-function penalties~\citep{hester2018deep,vecerik2017leveraging}, but only in the standard online RL setting with demonstrations. \citet{luo2019learning} learn a conservatively-extrapolated value function by enforcing a linear extrapolation property over the state-space, and a learned dynamics model to obtain policies for goal-reaching tasks. \citet{kakade2002approximately} proposed the CPI algorithm, that improves a policy conservatively in online RL.

Alternate prior approaches to offline RL estimate some sort of uncertainty to determine the trustworthiness of a Q-value prediction~\citep{kumar2019stabilizing,agarwal2019optimistic,levine2020offline}, typically using uncertainty estimation techniques from exploration in online RL~\citep{osband2016deep,jaksch2010near,osband2017posterior,burda2018exploration}. These methods have not been generally performant in offline RL~\citep{fujimoto2018off,kumar2019stabilizing,levine2020offline} due to the high-fidelity requirements of uncertainty estimates in offline RL~\citep{levine2020offline}. Robust MDPs~\citep{iyengar2005robust,ghavamzadeh2016safe,tamar2014scaling,nilim2004robustness} have been a popular theoretical abstraction for offline RL, but tend to be highly conservative in policy improvement. We expect CQL to be less conservative since CQL does not underestimate Q-values for all state-action tuples. Works on high confidence policy improvement~\citep{thomas2015high} provides safety guarantees for improvement but tend to be conservative.   
The gap-expanding property of CQL backups, shown in Theorem~\ref{thm:gap_amplify}, is related to how gap-increasing Bellman backup operators~\citep{bellemare2016increasing,lu2018general} are more robust to estimation error in online RL. 

\textbf{Theoretical results.} 
Our theoretical results (Theorems~\ref{thm:well_defined_obj}, \ref{thm:zeta_safe}) are related to prior work on safe policy improvement~\citep{laroche2017safe,ghavamzadeh2016safe}, and a direct comparison to Theorems 1 and 2 in \citet{laroche2017safe} suggests similar quadratic dependence on the horizon and an inverse square-root dependence on the counts. Our bounds improve over the $\infty$-norm bounds in \citet{ghavamzadeh2016safe}. Prior analyses have also focused on error propagation in approximate dynamic programming~\citep{farahmand2010error,chen2019information,xie2020q,scherrer2014approximate,kumar2019stabilizing}, and generally obtain bounds in terms of concentrability coefficients that capture the effect of distribution shift.

\vspace{-10pt}
\section{Experimental Evaluation}
\label{sec:experiments}
\vspace{-8pt}

We compare CQL to prior offline RL methods on a range of domains and dataset compositions, including continuous and discrete action spaces, state observations of varying dimensionality, and high-dimensional image inputs. We first evaluate actor-critic CQL, using CQL($\mathcal{H}$) from Algorithm~\ref{alg:practical_alg}, on continuous control datasets from the D4RL benchmark~\citep{d4rl}.
We compare to: prior offline RL methods that use a policy constraint -- BEAR~\citep{kumar2019stabilizing} and BRAC~\citep{wu2019behavior}; SAC~\citep{haarnoja}, an off-policy actor-critic method that we adapt to offline setting; and behavioral cloning (BC). 

\begin{table*}[h]
\captionsetup{font=small}
\centering
\fontsize{8}{8}\selectfont

\begin{tabular}{l|r|r|r|r|r||r}
\hline
\textbf{Task Name} & \textbf{SAC} & \textbf{BC} & \textbf{BEAR} & \textbf{BRAC-p} & \textbf{BRAC-v} & \textbf{CQL($\mathcal{H}$)}\\ \hline
halfcheetah-random &  30.5 & 2.1 & 25.5 & 23.5 & 28.1 & \textbf{35.4} \\
hopper-random & \textbf{11.3} & 9.8 & 9.5 & \textbf{11.1} & \textbf{12.0} & \textbf{10.8}\\
walker2d-random & 4.1 & 1.6 & \textbf{6.7} & 0.8 & 0.5 & \textbf{7.0}\\ \hline
halfcheetah-medium & -4.3 & 36.1 & 38.6 & \textbf{44.0} & \textbf{45.5} & \textbf{44.4}\\
walker2d-medium & 0.9 & 6.6  & 33.2 & 72.7 & \textbf{81.3} & {79.2}\\
hopper-medium & 0.8 & 29.0 & 47.6 & 31.2 & 32.3 & \textbf{58.0}\\ \hline
halfcheetah-expert & -1.9 & \textbf{107.0} & \textbf{108.2} & 3.8 & -1.1 & {104.8}\\
hopper-expert & 0.7 & \textbf{109.0} & \textbf{110.3} & 6.6 & 3.7 & \textbf{109.9} \\
walker2d-expert & -0.3 & 125.7 & 106.1 & -0.2 & -0.0 & \textbf{153.9} \\ \hline
halfcheetah-medium-expert & 1.8 & 35.8 & 51.7 & 43.8 & 45.3 & \textbf{62.4}\\
walker2d-medium-expert & 1.9 & 11.3 & 10.8 & -0.3 & 0.9 & \textbf{98.7}\\
hopper-medium-expert & 1.6 & \textbf{111.9} & 4.0 & 1.1 & 0.8 & \textbf{111.0}\\ \hline
halfcheetah-random-expert & 53.0 & 1.3 & 24.6 & 30.2 & 2.2 & \textbf{92.5}\\
walker2d-random-expert & 0.8 & 0.7 & 1.9 & 0.2 & 2.7 & \textbf{91.1}\\
hopper-random-expert & 5.6 & 10.1 & 10.1 & 5.8 & 11.1 & \textbf{110.5} \\ \hline
halfcheetah-mixed & -2.4 & 38.4 & 36.2 & \textbf{45.6} & \textbf{45.9} & \textbf{46.2}\\
hopper-mixed & 3.5 & 11.8 & 25.3 & 0.7 & 0.8 & \textbf{48.6}\\
walker2d-mixed & 1.9 & 11.3 & 10.8 & -0.3 & 0.9 & \textbf{26.7}\\
\hline
\end{tabular}
\vspace{-4pt}
\caption{\label{table:mujoco}{\small Performance of CQL($\mathcal{H}$) and prior methods on gym domains from D4RL, on the normalized return metric, averaged over 4 seeds. Note that CQL performs similarly or better than the best prior method with simple datasets, and greatly outperforms prior methods with complex distributions (``--mixed'', ``--random-expert'', ``--medium-expert'').}}
\normalsize
\vspace{-10pt}
\end{table*}

\textbf{Gym domains.} Results for the gym domains are shown in Table~\ref{table:mujoco}. The results for BEAR, BRAC, SAC, and BC are based on numbers reported by \citet{d4rl}. On the datasets generated from a single policy, marked as ``-random'', ``-expert'' and ``-medium'', CQL roughly matches or exceeds
the best prior methods, but by a small margin. However, on datasets that combine multiple policies (``-mixed'', ``-medium-expert'' and ``-random-expert''), that are more likely to be common in practical datasets, CQL outperforms prior methods by large margins, sometimes as much as \textbf{2-3x}.

\textbf{Adroit tasks.} The more complex Adroit~\citep{rajeswaran2018dapg} tasks (shown on the right) in D4RL require controlling a 24-DoF robotic hand, using limited data from human demonstrations. These tasks are
substantially more difficult than the gym tasks in terms of both the dataset composition and high dimensionality. Prior offline RL methods generally struggle to learn meaningful behaviors on 
\begin{wrapfigure}{r}{0.18\textwidth}
  \vspace{-13pt}
  \begin{center}
    \includegraphics[width=0.97\linewidth]{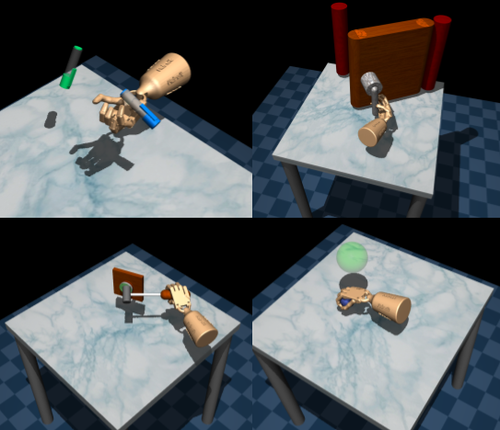}
  \end{center}
  \vspace{-19pt}
\end{wrapfigure}
these tasks, and the strongest baseline is BC. As shown in Table~\ref{table:adroit_antmaze}, CQL variants are the only methods that improve over BC, attaining scores that are \textbf{2-9x} those of the next best offline RL method. CQL($\rho$) with $\rho = \hat{\policy}^{k-1}$ (the previous policy) outperforms CQL($\mathcal{H}$) on a number of these tasks, due to the higher action dimensionality resulting in higher variance for the CQL($\mathcal{H}$) importance weights. Both variants outperform prior methods.

\begin{table}[H]
\captionsetup{font=small}
\vspace{-5pt}
\centering
\fontsize{8}{8}\selectfont
\begin{tabular}{l|l|r|r|r|r|r||r|r}
\hline
\textbf{Domain} & \textbf{Task Name} & \textbf{BC} & \textbf{SAC} & \textbf{BEAR} & \textbf{BRAC-p} & \textbf{BRAC-v} & \textbf{CQL($\mathcal{H}$)} & \textbf{CQL($\rho$)}\\ \hline
\multirow{6}*{AntMaze}
& antmaze-umaze & 65.0 & 0.0 & \textbf{73.0} & 50.0 & 70.0 & \textbf{74.0} & \textbf{73.5}\\
& antmaze-umaze-diverse  & 55.0 & 0.0 & 61.0 & 40.0 & 70.0 & \textbf{84.0} & 61.0\\
& antmaze-medium-play  & 0.0 & 0.0 & 0.0 & 0.0 & 0.0 & \textbf{61.2} & 4.6 \\
& antmaze-medium-diverse  & 0.0 & 0.0 & 8.0 & 0.0 & 0.0 & \textbf{53.7} & 5.1 \\
& antmaze-large-play & 0.0 & 0.0 & 0.0 & 0.0 & 0.0 & \textbf{15.8} &  3.2\\
& antmaze-large-diverse & 0.0 & 0.0 & 0.0 & 0.0 & 0.0 & \textbf{14.9} & 2.3 \\
\hline
\multirow{8}*{Adroit}
& pen-human  & 34.4 & 6.3 & -1.0 & 8.1 & 0.6 & 37.5 & \textbf{55.8}\\
& hammer-human & 1.5 & 0.5 & 0.3 & 0.3 & 0.2 & \textbf{4.4} & {2.1}\\
& door-human & 0.5 & 3.9 & -0.3 & -0.3 & -0.3 & \textbf{9.9} & \textbf{9.1} \\
& relocate-human & 0.0 & 0.0 & -0.3 & -0.3 & -0.3 & 0.20 & \textbf{0.35}\\
& pen-cloned  & \textbf{56.9} & 23.5 & 26.5 & 1.6 & -2.5 & 39.2 & 40.3\\
& hammer-cloned & 0.8 & 0.2 & 0.3 & 0.3 & 0.3 & 2.1 & \textbf{5.7} \\
& door-cloned & -0.1 & 0.0 & -0.1 & -0.1 & -0.1 & 0.4 & \textbf{3.5}\\
& relocate-cloned & \textbf{-0.1} & -0.2 & -0.3 & -0.3 & -0.3 & \textbf{-0.1} & \textbf{-0.1}\\
\hline
\multirow{3}*{Kitchen}
& kitchen-complete & 33.8 & 15.0 & 0.0 & 0.0 & 0.0 & \textbf{43.8} & 31.3\\
& kitchen-partial & 33.8 & 0.0 & 13.1 & 0.0 & 0.0 & \textbf{49.8} & \textbf{50.1} \\
& kitchen-undirected & 47.5 & 2.5 & 47.2 & 0.0 & 0.0 & \textbf{51.0} & \textbf{52.4} \\ \hline
\end{tabular}
\caption{\label{table:adroit_antmaze}{Normalized scores of all methods on AntMaze, Adroit, and kitchen domains from D4RL, averaged across 4 seeds. On the harder mazes, CQL is the \textit{only} method that attains non-zero returns, and is the only method to outperform simple behavioral cloning on Adroit tasks with human demonstrations.
We observed that the CQL($\rho$) variant, which avoids importance weights, trains more stably, with no sudden fluctuations in policy performance over the course of training, on the higher-dimensional Adroit tasks.}}
\normalsize
\vspace{-22pt}
\end{table}

\textbf{AntMaze.} These D4RL tasks require composing parts of suboptimal trajectories to form more optimal policies for reaching goals on a MuJoco Ant robot. 
Prior methods make some progress on the simpler U-maze, but only CQL is able to make meaningful progress  
on the much harder medium and large mazes, outperforming prior methods by a very wide margin.

\begin{wrapfigure}{r}{0.24\textwidth}
  \vspace{-19pt}
  \begin{center}
  \includegraphics[width=0.99\linewidth]{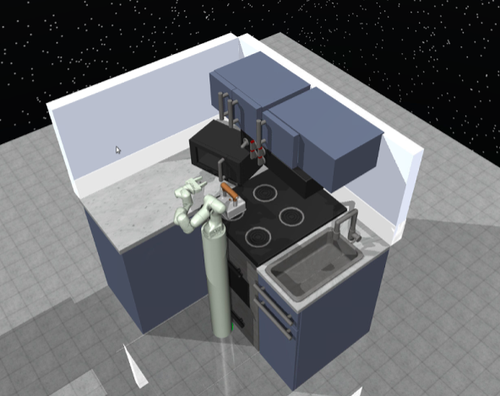}
  \end{center}
  \vspace{-19pt}
\end{wrapfigure}
\textbf{Kitchen tasks.} Next, we evaluate CQL on the Franka kitchen domain~\citep{gupta2019relay} from D4RL~\citep{d4rl_repo}.
The goal is to control a 9-DoF robot to manipulate multiple objects (microwave, kettle, etc.) \textit{sequentially}, in a single episode to reach a desired configuration, with only sparse 0-1 completion reward for every object that attains the target configuration. These tasks are especially challenging, since they require composing parts of trajectories, precise long-horizon manipulation, and handling human-provided teleoperation data. As shown in Table~\ref{table:adroit_antmaze}, CQL outperforms prior methods in this setting, and is the only method that outperforms behavioral cloning, attaining over \textbf{40\%} success rate on all tasks.

\textbf{Offline RL on Atari games.} Lastly, we evaluate a discrete-action Q-learning variant of CQL (Algorithm~\ref{alg:practical_alg}) on offline, image-based Atari games~\citep{bellemare2013arcade}. We compare CQL to REM~\citep{agarwal2019optimistic} and QR-DQN~\citep{dabney2018distributional} on the five Atari tasks (Pong, Breakout, Qbert, Seaquest and Asterix) that are evaluated in detail by \citet{agarwal2019optimistic}, using the dataset released by the authors. 

\begin{figure*}[h]
\vspace{-12pt}
\begin{center}
  \includegraphics[width=0.23\linewidth]{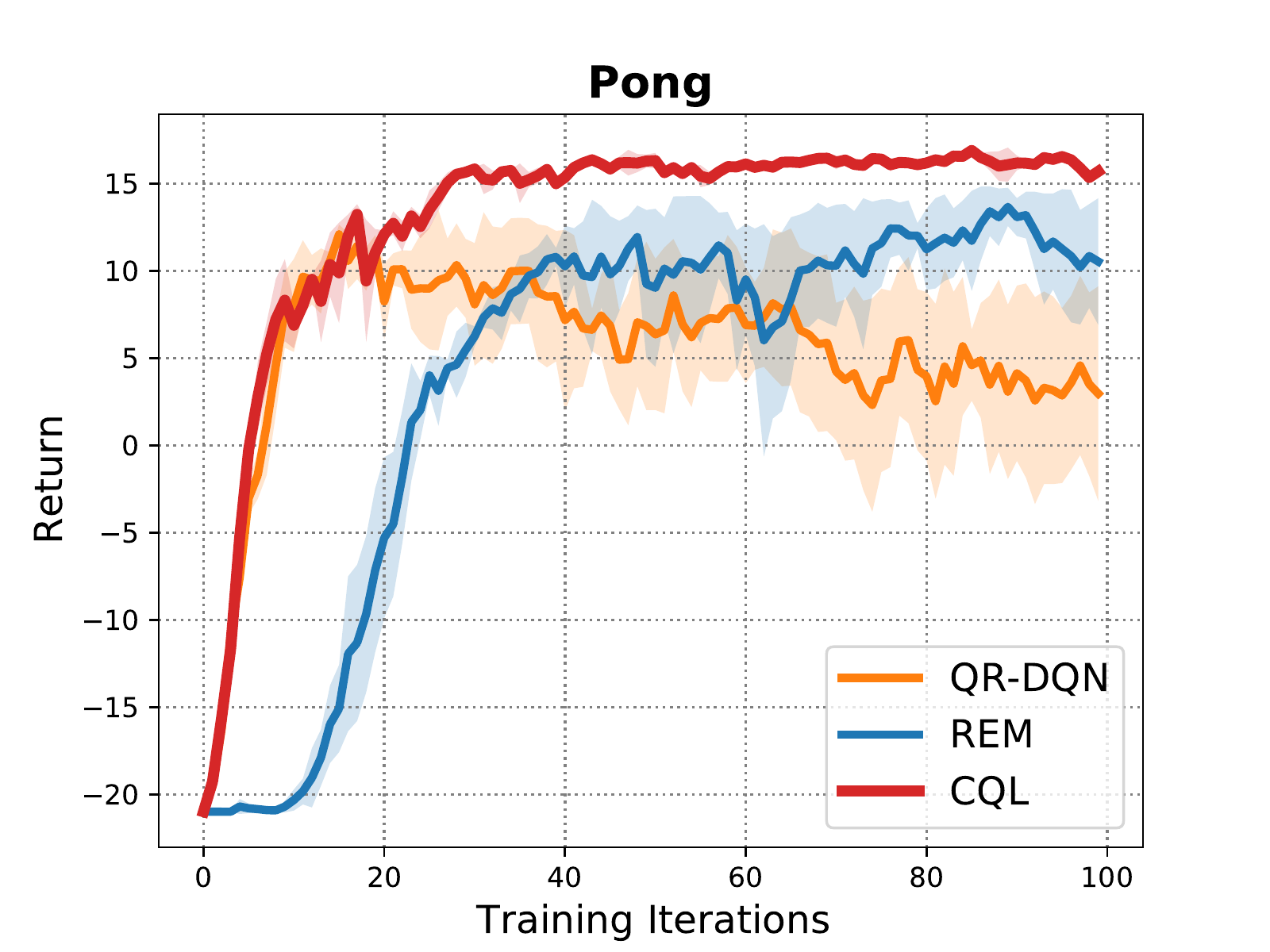}  
  \includegraphics[width=.23\linewidth]{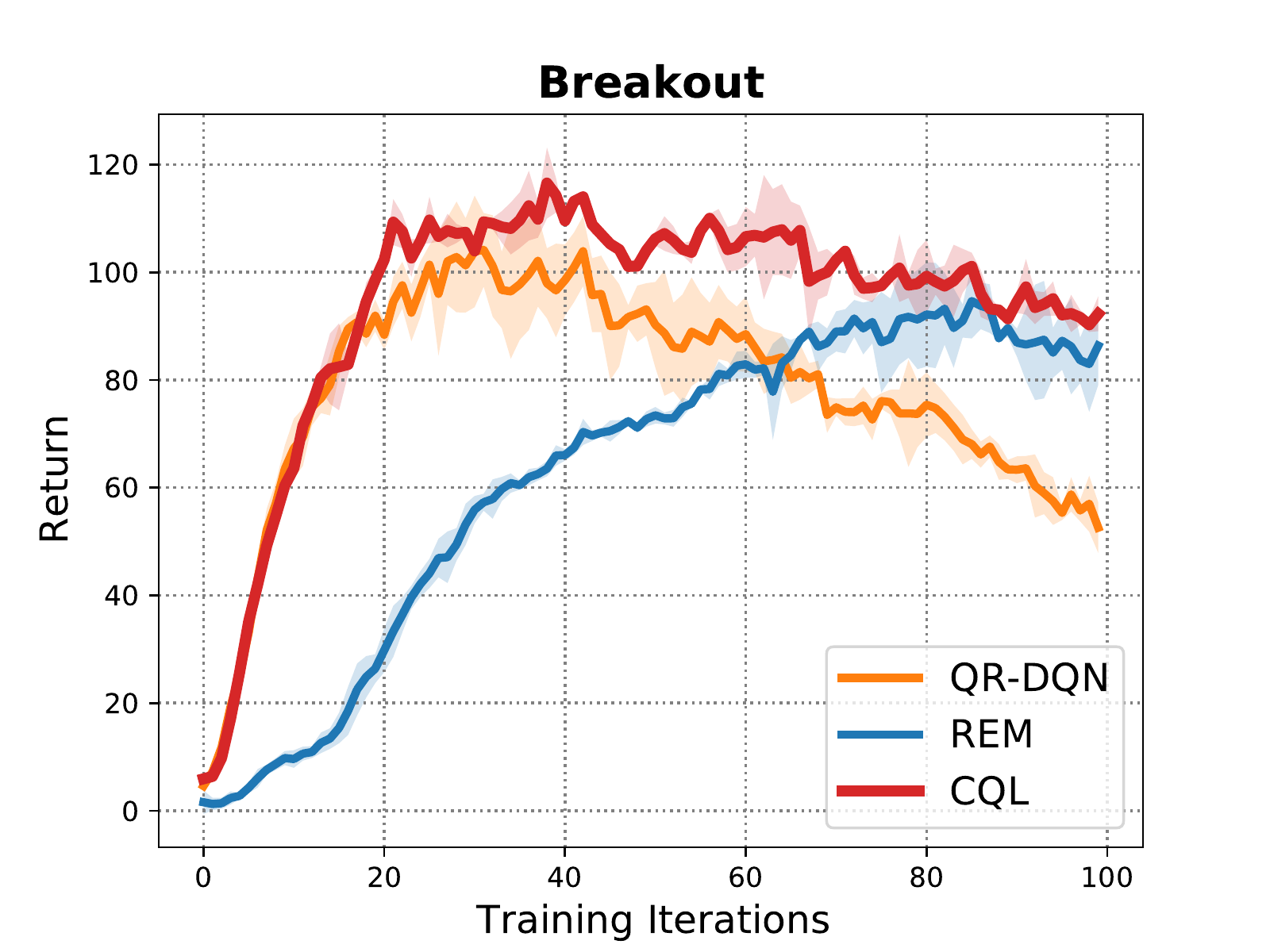}  
  \includegraphics[width=.23\linewidth]{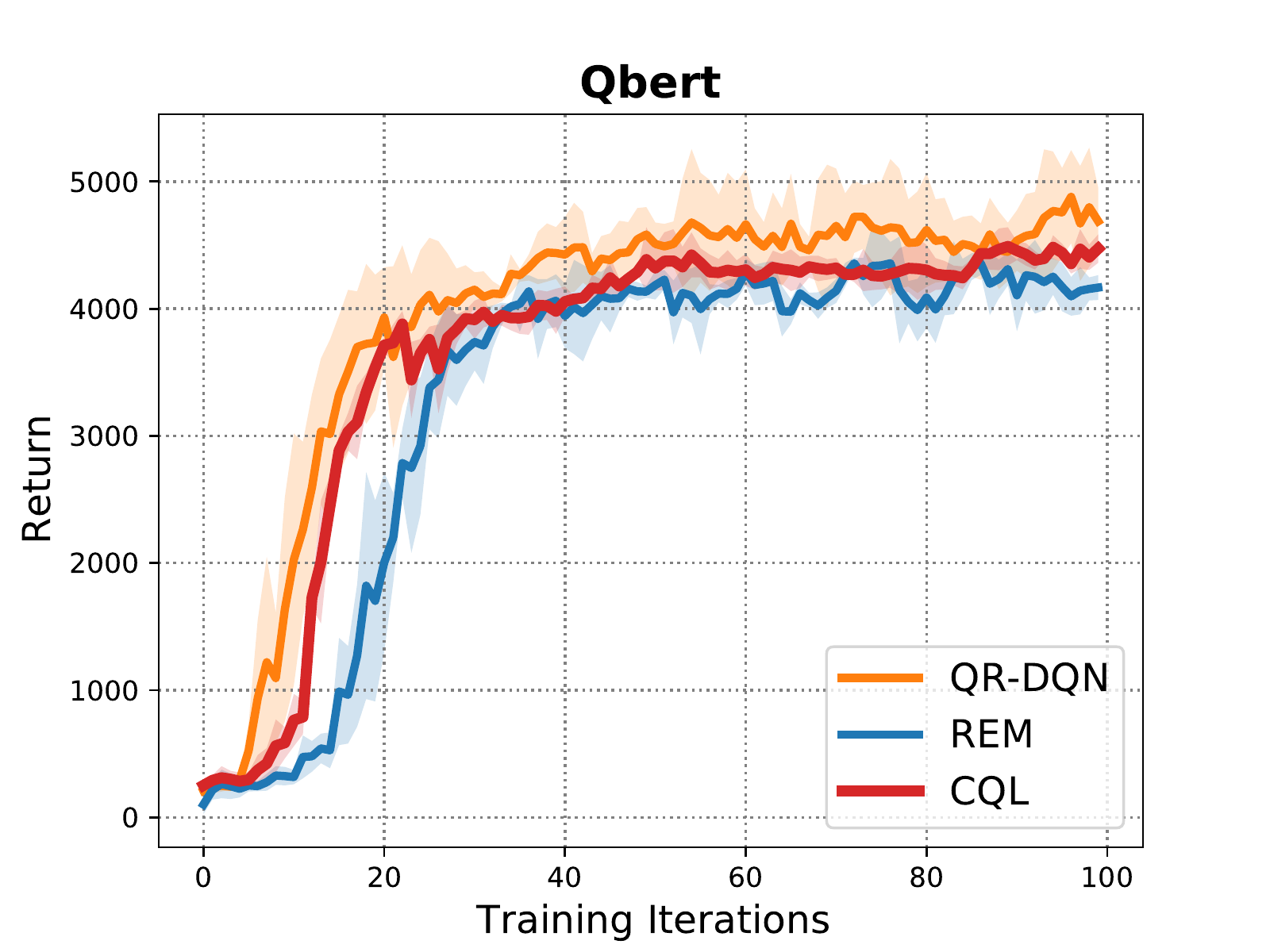}  
  \includegraphics[width=.23\linewidth]{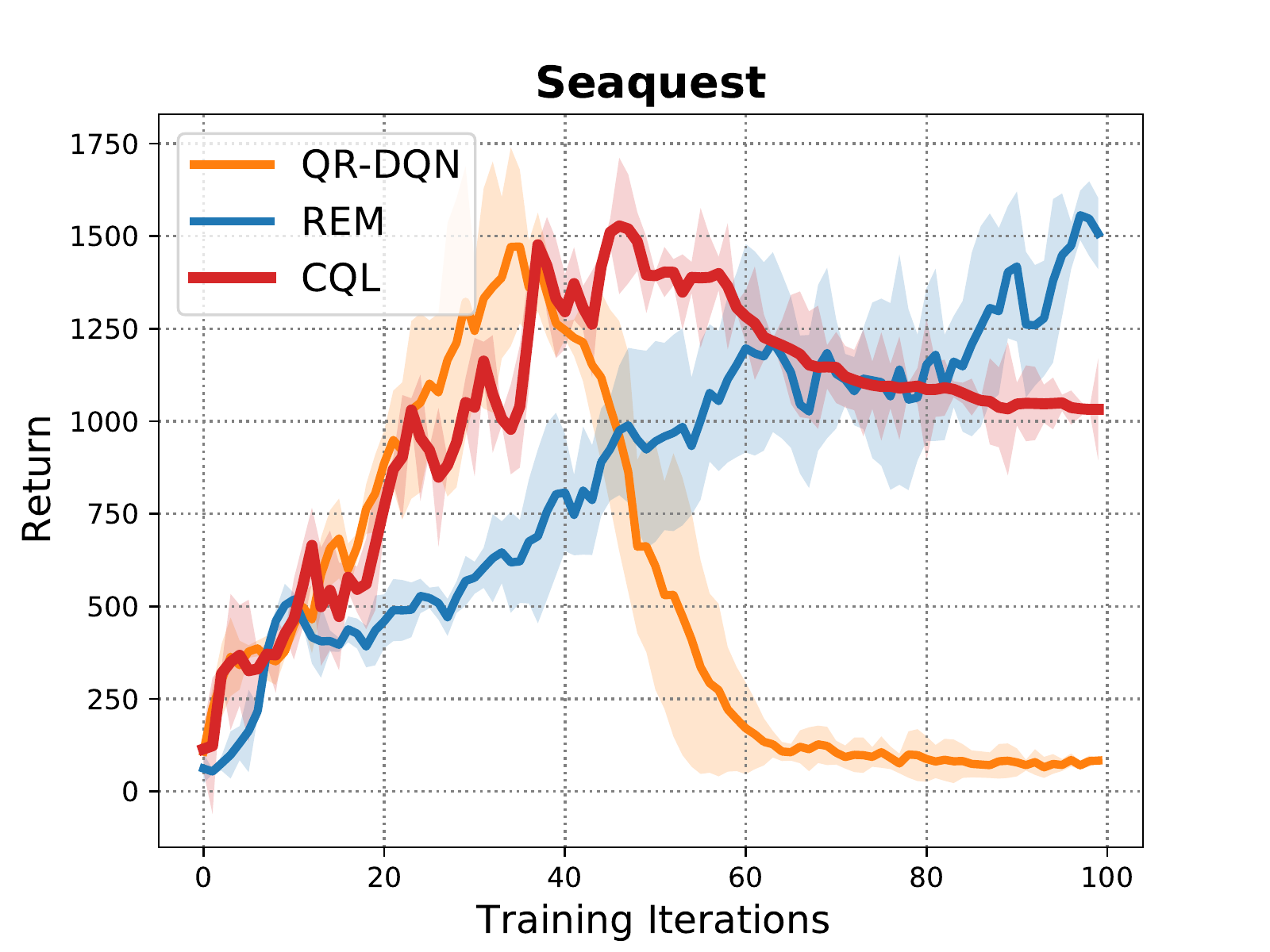}  
\end{center}
\caption{\label{fig:cql_20m_atari}{\footnotesize Performance of CQL, QR-DQN and REM as a function of training steps (x-axis) in setting \textbf{(1)} when provided with only the first 20\% of the samples of an online DQN run. Note that CQL is able to learn stably on 3 out of 4 games, and its performance does not degrade as steeply as QR-DQN on Seaquest$^*$.}}
\vspace{-12pt}
\end{figure*}
Following the evaluation protocol of \citet{agarwal2019optimistic}, we evaluated on two types of datasets, both of which were generated from the DQN-replay dataset, released by~\citep{agarwal2019optimistic}:
\textbf{(1)} a dataset consisting of the first 20\% of the samples observed by an online DQN agent and 
\textbf{(2)} datasets consisting of only 1\%  and 10\% of all samples observed by an online DQN agent (Figures 6 and 7 in \citep{agarwal2019optimistic}). In setting \textbf{(1)}, shown in Figure~\ref{fig:cql_20m_atari}, CQL generally achieves similar or better performance throughout as QR-DQN and REM. When only using only 1\% or 10\% of the data, in setting \textbf{(2)} (Table~\ref{table:atari_reduced_size}),  CQL 
\begin{wraptable}{r}{6.0cm}
\captionsetup{font=small}
    \centering
    \vspace{-7pt}
    \fontsize{8}{8}\selectfont
    \begin{tabular}{l|r|r||r}
    \hline
        \textbf{Task Name} & \textbf{QR-DQN} & \textbf{REM} & \textbf{CQL($\mathcal{H}$)} \\
        \hline
         Pong (1\%) & -13.8 & -6.9 & \textbf{19.3} \\
         Breakout & 7.9 & 11.0 & \textbf{61.1} \\
         Q*bert & 383.6 & 343.4 & \textbf{14012.0} \\
         Seaquest & 672.9 & 499.8 & \textbf{779.4} \\
         Asterix  & 166.3 & 386.5 & \textbf{592.4}\\
         \hline
         \hline
         Pong (10\%) & 15.1 & 8.9 & \textbf{18.5} \\
         Breakout & 151.2 & 86.7 & \textbf{269.3} \\
         Q*bert & 7091.3 & 8624.3 & \textbf{13855.6}\\
         Seaquest & 2984.8 & \textbf{3936.6} & 3674.1 \\
         Asterix & \textbf{189.2} & 75.1 & 156.3 \\
         \hline
    \end{tabular}
    \vspace{-4pt}
    \caption{{CQL, REM and QR-DQN in setting \textbf{(1)} with 1\% data (top), and 10\% data (bottom). CQL drastically outperforms prior methods with 1\% data, and usually attains better performance with 10\% data.}}
    \normalsize
    \label{table:atari_reduced_size}
    \vspace{-17pt}
\end{wraptable}
substantially outperforms REM and QR-DQN, especially in the harder 1\% condition, achieving \textbf{36x} and \textbf{6x} times the return of the best prior method on Q*bert and Breakout, respectively.

\textbf{Analysis of CQL.} Finally, we perform empirical evaluation to verify that CQL indeed lower-bounds the value function, thus verifying Theorems~\ref{thm:cql_underestimates}, Appendix~\ref{thm:policy_eval_func_approx} empirically. To this end, we estimate the average value of the learned policy predicted by CQL, $\E_{\bs \sim \mathcal{D}}[\hat{V}^k(\bs)]$, and report the difference against the actual discounted return of the policy $\policy^{k}$ in Table~\ref{table:cql_lower_bound}. We also estimate these values for baselines, including the minimum predicted Q-value under an ensemble~\citep{haarnoja,fujimoto2018addressing}
of Q-functions with varying ensemble sizes, which is a standard technique to prevent overestimed Q-values~\citep{fujimoto2018addressing,haarnoja,hasselt2010double} and BEAR~\citep{kumar2019stabilizing}, a policy constraint method. The results show that CQL learns a lower bound for all three tasks, whereas the baselines are prone to overestimation. We also evaluate a variant of CQL that uses Equation~\ref{eqn:objective_1}, and observe that the resulting values are lower (that is, underestimate the true values) as compared to CQL($\mathcal{H}$). This provides empirical evidence that CQL($\mathcal{H}$) attains a tighter lower bound than the point-wise bound in Equation~\ref{eqn:objective_1}, as per Theorem~\ref{thm:cql_underestimates}.

\begin{table}[h]
\captionsetup{font=small}
\centering
\vspace{-5pt}
    \fontsize{8}{8}\selectfont
    \begin{tabular}{l|r|r||r|r|r|r|r}
    \hline
        \textbf{Task Name} & \textbf{CQL($\mathcal{H})$} & \textbf{CQL (Eqn.~\ref{eqn:objective_1})} & \textbf{Ensemble(2)}  & \textbf{Ens.(4)} & \textbf{Ens.(10)} & \textbf{Ens.(20)} & \textbf{BEAR} \\
        \hline
        hopper-medium-expert & \textbf{-43.20} & -151.36 & 3.71e6 & 2.93e6 & 0.32e6 & 24.05e3 & 65.93 \\
        hopper-mixed & \textbf{-10.93} & -22.87 & 15.00e6 &  59.93e3 & 8.92e3 & 2.47e3 & 1399.46 \\
        hopper-medium & \textbf{-7.48} & -156.70 & 26.03e12 & 437.57e6 & 1.12e12 & 885e3 & 4.32 \\
        \hline
    \end{tabular}
    \caption{{Difference between policy values predicted by each algorithm and the true policy value for CQL, a variant of CQL that uses Equation~\ref{eqn:objective_1}, the minimum of an ensemble of varying sizes, and BEAR~\citep{kumar2019stabilizing} on three D4RL datasets. CQL is the only method that lower-bounds the actual return (i.e., has negative differences), and CQL($\mathcal{H})$ is much less conservative than CQL (Eqn.~\ref{eqn:objective_1}).}}
    \normalsize
    \vspace{-10pt}
    \label{table:cql_lower_bound}
\end{table}

We also present an empirical analysis to show that Theorem~\ref{thm:gap_amplify}, that CQL is gap-expanding, holds in practice in Appendix~\ref{app:gap_amplify}, and present an ablation study on various design choices used in CQL in Appendix~\ref{app:additional_results}.

\vspace{-8pt}
\section{Discussion}
\vspace{-5pt}
We proposed conservative Q-learning (CQL), an algorithmic framework for offline RL that learns a lower bound on the policy value.
Empirically, we demonstrate that CQL outperforms prior offline RL methods on a wide range of offline RL benchmark tasks, including complex control tasks and tasks with raw image observations. In many cases, the performance of CQL is substantially better than the best-performing prior methods, exceeding their final returns by 2-5x.
The simplicity and efficacy of CQL make it a promising choice for a wide range of real-world offline RL problems. However, a number of challenges remain. While we prove that CQL learns lower bounds on the Q-function in the tabular, linear, and a subset of non-linear function approximation cases, a rigorous theoretical analysis of CQL with deep neural nets, is left for future work. Additionally, offline RL methods are liable to suffer from overfitting in the same way as standard supervised methods, so another important challenge for future work is to devise simple and effective early stopping methods, analogous to validation error in supervised learning.

\section*{Acknowledgements}
We thank Mohammad Norouzi, Oleh Rybkin, Anton Raichuk, Vitchyr Pong and anonymous reviewers from the Robotic AI and Learning Lab at UC Berkeley for their feedback on an earlier version of this paper. We thank Rishabh Agarwal for help with the Atari QR-DQN/REM codebase and for sharing baseline results. This research was funded by the DARPA Assured Autonomy program, and compute support from Google, Amazon, and NVIDIA.

\bibliography{main.bib}
\bibliographystyle{plainnat}

\newpage
\appendix
\part*{Appendices}

\section{Discussion of CQL Variants}
\label{app:cql_variants}
We derive several variants of CQL in Section~\ref{sec:framework}. Here, we discuss these variants on more detail and describe their specific properties. We first derive the variants: CQL($\mathcal{H}$), CQL($\rho$), and then present another variant of CQL, which we call CQL(var). This third variant has strong connections to distributionally robust optimization~\citep{namkoong2017variance}.

\textbf{CQL($\mathcal{H}$).} In order to derive CQL($\mathcal{H}$), we substitute $\mathcal{R} = \mathcal{H}(\mu)$, and solve the optimization over $\mu$ in closed form for a given Q-function. For an optimization problem of the form:
\begin{equation*}
    \max_{\mu}~~ \E_{\bx \sim \mu(\bx)}[f(\bx)] + \mathcal{H}(\mu)~~~ \text{s.t.}~~~ \sum_{\bx} \mu(\bx) = 1,~ \mu({\bx}) \geq 0~ \forall \bx,
\end{equation*}
the optimal solution is equal to $\mu^*(\bx) = \frac{1}{Z} \exp(f(\bx))$, where $Z$ is a normalizing factor. Plugging this into Equation~\ref{eqn:cql_framework}, we exactly obtain Equation~\ref{eqn:practical_objective}.

\textbf{CQL($\rho$).} In order to derive CQL($\rho$), we follow the above derivation, but our regularizer is a KL-divergence regularizer instead of entropy.
\begin{equation*}
    \small{\max_{\mu}~~ \E_{\bx \sim \mu(\bx)}[f(\bx)] + D_{\mathrm{KL}}(\mu || \rho)~~~ \text{s.t.}~~~ \sum_{\bx} \mu(\bx) = 1,~ \mu({\bx}) \geq 0~ \forall \bx}.
\end{equation*}
The optimal solution is given by, $\mu^*(\bx) = \frac{1}{Z} \rho(\bx) \exp(f(\bx))$, where $Z$ is a normalizing factor. Plugging this back into the CQL family (Equation~\ref{eqn:cql_framework}), we obtain the following objective for training the Q-function (modulo some normalization terms):
\begin{equation}
    \small{\min_{Q}~ \alpha \E_{\bs \sim d^\behavior(\bs)}\left[\E_{\ba \sim \rho(\ba|\bs)} \left[Q(\bs, \ba) \frac{\exp(Q(\bs, \ba))}{Z}\right] - \E_{\ba \sim \behavior(\ba|\bs)}\left[Q(\bs, \ba)\right]\right] + \frac{1}{2}\!\E_{\bs, \ba, \bs' \sim \mathcal{D}}\left[\left(Q - \bellman^{\policy_k} \hat{Q}^{k} \right)^2 \right]\!.}
    \label{eqn:cql_rho_objective}
\end{equation}

\textbf{CQL(var).} Finally, we derive a CQL variant that is inspired from the perspective of distributionally robust optimization (DRO)~\citep{namkoong2017variance}. This version penalizes the variance in the Q-function across actions at all states $\bs$, under some action-conditional distribution of our choice. In order to derive a canonical form of this variant, we invoke an identity from \citet{namkoong2017variance}, which helps us simplify Equation~\ref{eqn:cql_framework}. To start, we define the notion of ``robust expectation'': for any function $f(\bx)$, and any empirical distribution $\hat{P}(\bx)$ over a dataset $\{ \bx_1, \cdots, \bx_N\}$ of $N$ elements, the ``robust'' expectation defined by: 
\begin{equation*}
    R_N(\hat{P}) := \max_{\mu(\bx)} ~~\E_{\bx \sim \mu(\bx)}[f(\bx)] \text{~~~s.t.~~~} D_{f}(\mu(\bx), \hat{P}(\bx)) \leq \frac{\delta}{N},
\end{equation*}    
can be approximated using the following upper-bound:
\begin{equation}
    R_N(\hat{P}) \leq \E_{\bx \sim \hat{P}(\bx)}[f(\bx)] + \sqrt{\frac{2 \delta~ \text{var}_{\hat{P}(\bx)}(f(\bx))}{N}},
    \label{eqn:robust_expectation}
\end{equation}
where the gap between the two sides of the inequality decays inversely w.r.t. to the dataset size, $\mathcal{O}(1/N)$. By using Equation~\ref{eqn:robust_expectation} to simplify Equation~\ref{eqn:cql_framework}, we obtain an objective for training the Q-function that penalizes the variance of Q-function predictions under the distribution $\hat{P}$. 
\begin{multline}
    \min_{Q}~~ \frac{1}{2}~ \E_{\bs, \ba, \bs' \sim \mathcal{D}}\left[\left(Q - \bellman^{\policy_k} \hat{Q}^{k} \right)^2 \right] + \alpha \E_{\bs \sim d^\behavior(\bs)}\left[\sqrt{\frac{\text{var}_{\hat{P}(\ba|\bs)}\left( Q(\bs, \ba) \right)}{d^\behavior(s) |\mathcal{D}|}} \right] \\ 
    + \alpha \E_{s \sim d^\behavior(\bs)}\left[ \E_{\hat{P}(\ba|\bs)}[Q(\bs, \ba)] - \E_{\behavior(\ba|\bs)}[Q(\bs, \ba)] \right]
    \label{eqn:variance_regularized_again}
\end{multline}
The only remaining decision is the choice of $\hat{P}$, which can be chosen to be the inverse of the empirical action distribution in the dataset, $\hat{P}(\ba|\bs) \propto \frac{1}{\hat{D}(\ba|\bs)}$, or even uniform over actions, $\hat{P}(\ba|\bs) = \text{Unif}(\ba)$, to obtain this variant of variance-regularized CQL.

\section{Discussion of Gap-Expanding Behavior of CQL Backups}
\label{app:gap_amplify}

In this section, we discuss in detail the consequences of the gap-expanding behavior of CQL backups over prior methods based on policy constraints that, as we show in this section, may not exhibit such gap-expanding behavior in practice. To recap, Theorem~\ref{thm:gap_amplify} shows that the CQL backup operator increases the difference between expected Q-value at in-distribution ($\ba \sim \behavior(\ba|\bs)$) and out-of-distribution ($\ba \text{~s.t.~} \frac{\mu_k(\ba|\bs)}{\behavior(\ba|\bs)} << 1$) actions. We refer to this property as the gap-expanding property of the CQL update operator.

\textbf{Function approximation may give rise to erroneous Q-values at OOD actions.} We start by discussing the behavior of prior methods based on policy constraints~\citep{kumar2019stabilizing,fujimoto2018off,jaques2019way,wu2019behavior} in the presence of function approximation.
To recap, because computing the target value requires $\E_\policy[\hat{Q}(\bs,\ba)]$, constraining $\policy$ to be close to $\behavior$ will avoid evaluating $\hat{Q}$ on OOD actions. These methods typically do not impose any further form of regularization on the learned Q-function.
Even with policy constraints, because function approximation used to represent the Q-function, learned Q-values at two distinct state-action pairs are coupled together. As prior work has argued and shown~\citep{achiam2019towards,fu2019diagnosing,kumar2020discor}, the ``generalization'' or the coupling effects of the function approximator may be heavily influenced by the properties of the data distribution~\citep{fu2019diagnosing,kumar2020discor}. For instance, \citet{fu2019diagnosing} empirically shows that when the dataset distribution is narrow (i.e. state-action marginal entropy, $\mathcal{H}(d^\behavior(\bs, \ba))$, is low~\citep{fu2019diagnosing}), the coupling effects of the Q-function approximator can give rise to incorrect Q-values at different states, though this behavior is absent without function approximation, and is not as severe with high-entropy (e.g. Uniform) state-action marginal distributions.

In offline RL, we will shortly present empirical evidence on high-dimensional MuJoCo tasks showing that certain dataset distributions, $\mathcal{D}$, may cause the learned Q-value at an OOD action $\ba$ at a state $\bs$, to in fact take on high values than Q-values at in-distribution actions at intermediate iterations of learning. This problem persists even when a large number of samples (e.g. $1M$) are provided for training, and the agent cannot correct these errors due to no active data collection.  

Since actor-critic methods, including those with policy constraints, use the learned Q-function to train the policy, in an iterative online policy evaluation and policy improvement cycle, as discussed in Section~\ref{sec:background}, the errneous Q-function may push the policy towards OOD actions, especially when no policy constraints are used. Of course, policy constraints should prevent the policy from choosing OOD actions, however, as we will show that in certain cases, policy constraint methods might also fail to prevent the effects on the policy due to incorrectly high Q-values at OOD actions.

\textbf{How can CQL address this problem?} As we show in Theorem~\ref{thm:gap_amplify}, the difference between expected Q-values at in-distribution actions and out-of-distribution actions is expanded by the CQL update. This property is a direct consequence of the specific nature of the CQL regularizer -- that maximizes Q-values under the dataset distribution, and minimizes them otherwise. This difference depends upon the choice of $\alpha_k$, which can directly be controlled, since it is a free parameter. Thus, by effectively controlling $\alpha_k$, CQL can push down the learned Q-value at out-of-distribution actions as much is desired, correcting for the erroneous overestimation error in the process.

\textbf{Empirical evidence on high-dimensional benchmarks with neural networks.}  
We next empirically demonstrate the existence of of such Q-function estimation error on high-dimensional MuJoCo domains when deep neural network function approximators are used with stochastic optimization techniques. In order to measure this error, we plot the difference in expected Q-value under actions sampled from the behavior distribution, $\ba \sim \behavior(\ba|\bs)$, and the maximum Q-value over actions sampled from a uniformly random policy, $\ba \sim \text{Unif}(\ba|\bs)$. That is, we plot the quantity
\begin{equation}
\label{eqn:delta_eqn}
    \hat{\Delta}^k = \E_{\bs, \ba \sim \mathcal{D}}\left[\max_{\ba'_1, \cdots, \ba'_N \sim \text{Unif}(\ba')}[\hat{Q}^k(\bs, \ba')]- \hat{Q}^k(\bs, \ba)\right]
\end{equation}
over the iterations of training, indexed by $k$. This quantity, intuitively, represents an estimate of the ``advantage'' of an action $\ba$, under the Q-function, with respect to the optimal action $\max_{\ba'} \hat{Q}^k(\bs, \ba')$. Since, we cannot perform exact maximization over the learned Q-function in a continuous action space to compute $\Delta$, we estimate it via sampling described in Equation~\ref{eqn:delta_eqn}.

We present these plots in Figure~\ref{fig:delta_plots} on two datasets: hopper-expert and hopper-medium. The expert dataset is generated from a near-deterministic, expert policy, exhibits a narrow coverage of the state-action space, and limited to only a few directed trajectories. On this dataset, we find that $\hat{\Delta}^k$ is always positive for the policy constraint method (Figure~\ref{fig:delta_plots}(a)) and increases during training -- note, the continuous rise in $\hat{\Delta}^k$ values, in the case of the policy-constraint method, shown in Figure~\ref{fig:delta_plots}(a). This means that even if the dataset is generated from an expert policy, and policy constraints correct target values for OOD actions,
incorrect Q-function generalization may make an out-of-distribution action appear promising. For the more stochastic hopper-medium dataset, that consists of a more diverse set of trajectories, shown in Figure~\ref{fig:delta_plots}(b), we still observe that $\hat{\Delta}^k > 0$ for the policy-constraint method, however, the relative magnitude is smaller than hopper-expert.

In contrast, Q-functions learned by CQL, generally satisfy $\hat{\Delta}^k < 0$, as is seen  and these values are clearly smaller than those for the policy-constraint method. This provides some empirical evidence for Theorem~\ref{thm:gap_amplify}, in that, the maximum Q-value at a randomly chosen action from the uniform distribution the action space is smaller than the Q-value at in-distribution actions.

On the hopper-expert task, as we show in Figure~\ref{fig:delta_plots}(a) (right), we eventually observe an ``unlearning'' effect, in the policy-constraint method where the policy performance deteriorates after a extra iterations in training. This ``unlearning'' effect is similar to what has been observed when standard off-policy Q-learning algorithms without any policy constraint are used in the offline regime~\citep{kumar2019stabilizing,levine2020offline}, on the other hand this effect is absent in the case of CQL, even after equally many training steps. The performance in the more-stochastic hopper-medium dataset fluctuates, but does not deteriorate suddenly.

To summarize this discussion, we concretely observed the following points via empirical evidence:
\begin{itemize}
\vspace{-10pt}
    \item CQL backups are gap expanding in practice, as justified by the negative $\hat{\Delta}^k$ values in Figure~\ref{fig:delta_plots}.
    \item Policy constraint methods, that do not impose any regularization on the Q-function may observe highly positive $\hat{\Delta}^k$ values during training, especially with narrow data distributions, indicating that gap-expansion may  be absent.
    \item When $\hat{\Delta}^k$ values continuously grow during training, the policy might eventually suffer from an unlearning effect~\citep{levine2020offline}, as shown in Figure~\ref{fig:delta_plots}(a).
    \vspace{-10pt}
\end{itemize}

\begin{figure}
    \begin{subfigure}[h]{0.49\linewidth}
      \centering
      \includegraphics[width=0.47\linewidth]{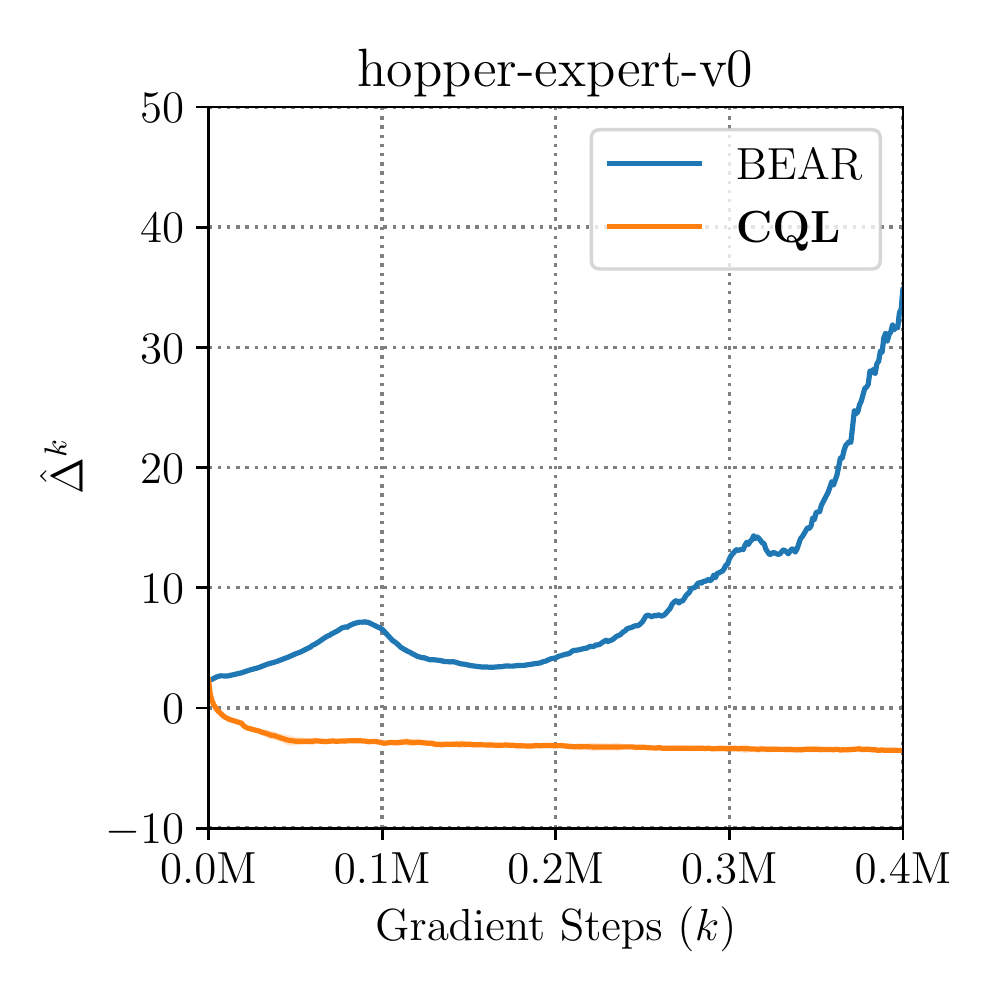}
      \includegraphics[width=0.47\linewidth]{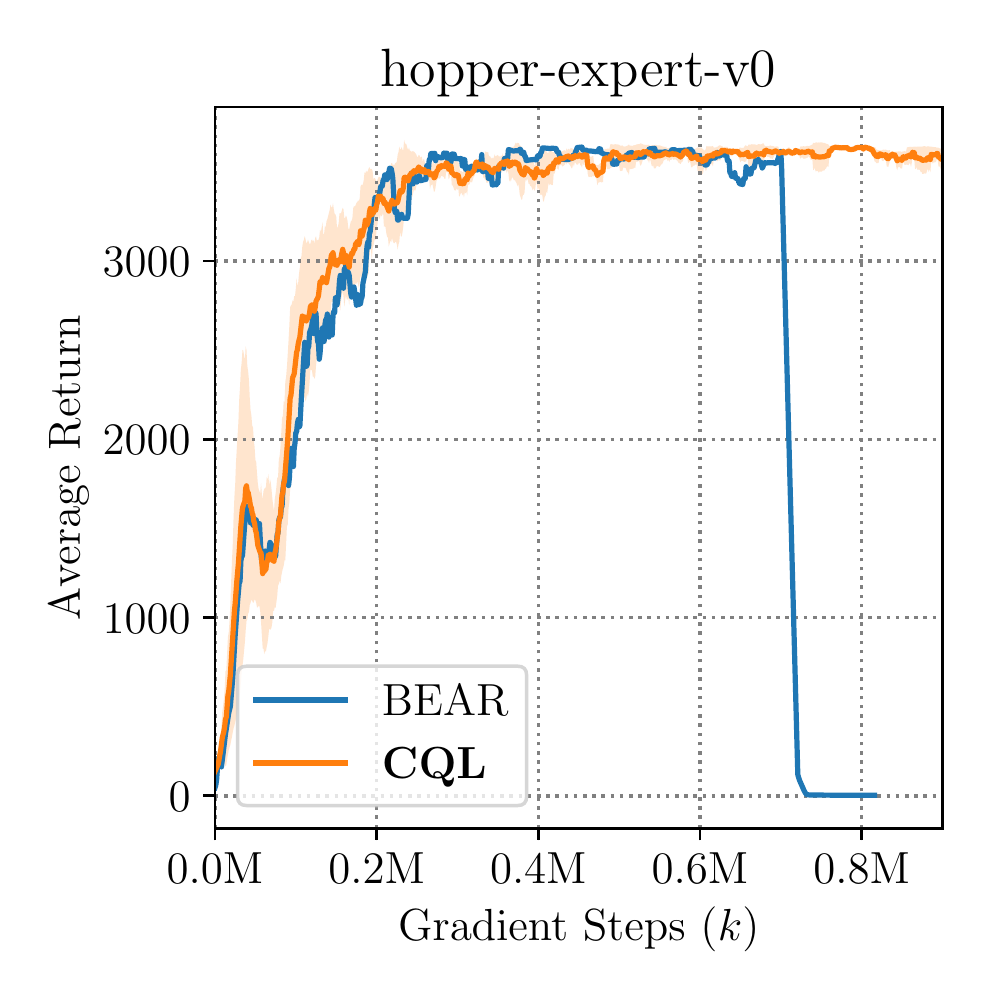}
      \caption{hopper-expert-v0}
    \end{subfigure}
    ~
    \begin{subfigure}[h]{0.49\linewidth}
      \centering
      \includegraphics[width=0.47\linewidth]{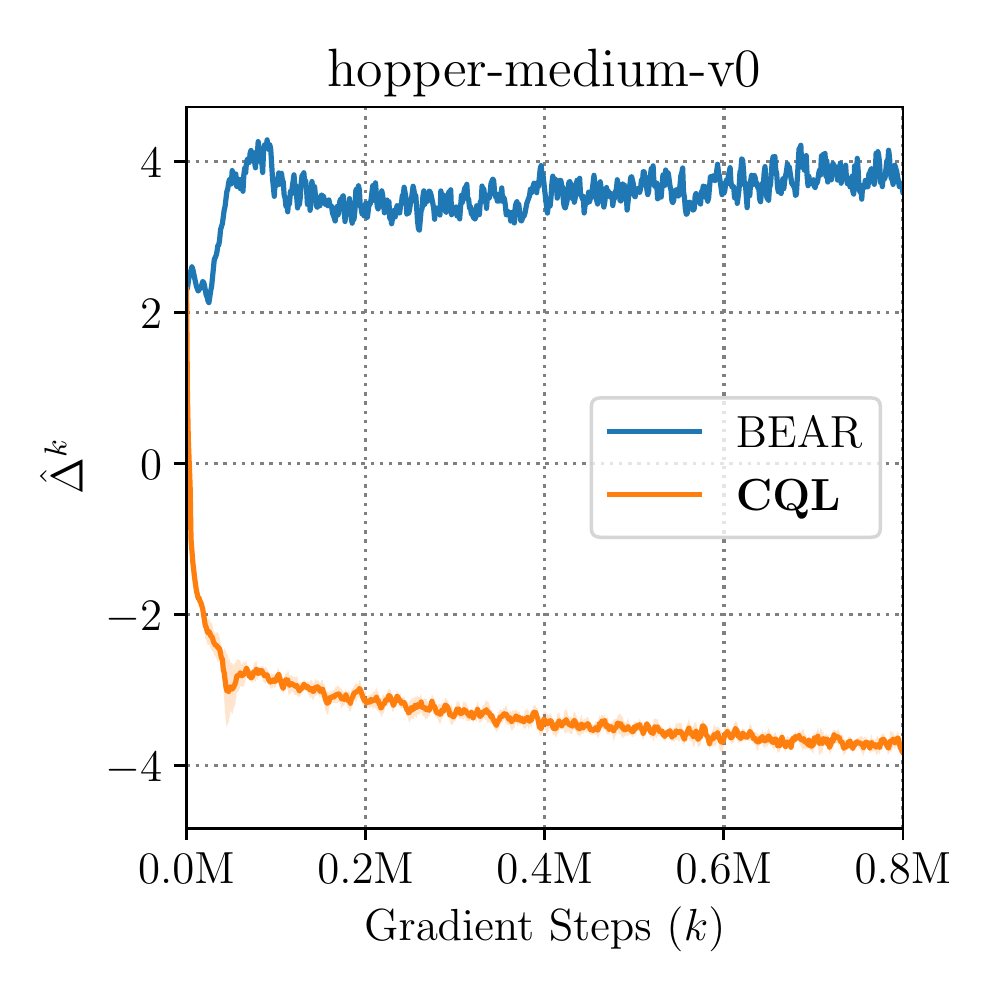}
      \includegraphics[width=0.47\linewidth]{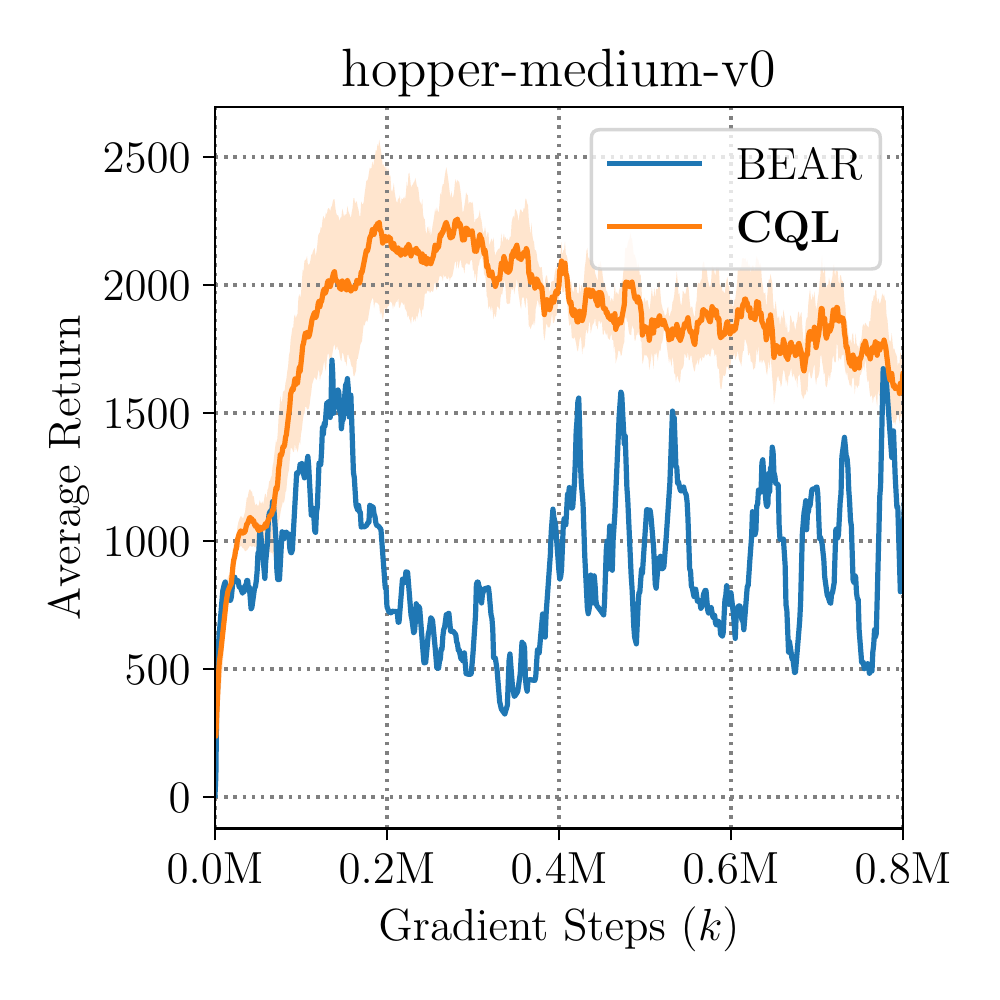}
      \caption{hopper-medium-v0}
    \end{subfigure}
    \caption{$\Delta^k$ as a function of training iterations for hopper-expert and hopper-medium datasets. Note that CQL (left) generally has negative values of $\Delta$, whereas BEAR (right) generally has positive $\Delta$ values, which also increase during training with increasing $k$ values.}
    \label{fig:delta_plots}
\end{figure}

\section{Theorem Proofs}
\label{app:missing_proofs}
In this section, we provide proofs of the theorems in Sections~\ref{sec:policy_eval} and \ref{sec:framework}. We first redefine notation for clarity and then provide the proofs of the results in the main paper.

\textbf{Notation.} Let $k \in \mathbb{N}$ denote an iteration of policy evaluation (in Section~\ref{sec:policy_eval}) or Q-iteration (in Section~\ref{sec:framework}). In an iteration $k$, the objective -- Equation~\ref{eqn:modified_policy_eval} or Equation~\ref{eqn:cql_framework} -- is optimized using the previous iterate (i.e. $\hat{Q}^{k-1}$) as the target value in the backup. $Q^k$ denotes the true, tabular Q-function iterate in the MDP, without any correction. In an iteration, say $k+1$, the current tabular Q-function iterate, $Q^{k+1}$ is related to the previous tabular Q-function iterate ${Q}^k$ as: $Q^{k+1} = \bellman^{\policy} Q^k$ (for policy evaluation) or $Q^{k+1} = \bellman^{\policy_k} Q^k$ (for policy learning). Let $\hat{Q}^k$ denote the $k$-th Q-function iterate obtained from CQL. Let $\hat{V}^k$ denote the value function, $\hat{V}^k := \E_{\ba \sim \policy(\ba|\bs)}[\hat{Q}^k(\bs, \ba)]$.   

\textbf{A note on the value of $\alpha$.} Before proving the theorems, we remark that while the statements of Theorems~\ref{thm:cql_underestimates}, \ref{thm:min_q_underestimates} and \ref{thm:policy_eval_func_approx} (we discuss this in Appendix~\ref{app:additional_theory}) show that CQL produces lower bounds if $\alpha$ is larger than some threshold, so as to overcome either sampling error (Theorems~\ref{thm:cql_underestimates} and \ref{thm:min_q_underestimates}) or function approximation error (Theorem~\ref{thm:policy_eval_func_approx}). While the optimal $\alpha_k$ in some of these cases depends on the current Q-value, $\hat{Q}^k$, we can always choose a worst-case value of $\alpha_k$ by using the inequality $\hat{Q}^k \leq 2 R_{\max}/(1 - \gamma)$, still guaranteeing a lower bound. If it is unclear why the learned Q-function $\hat{Q}^k$ should be bounded, we can always clamp the Q-values if they go outside $\left[ \frac{-2 R_{\max}}{1 - \gamma}, \frac{2 R_{\max}}{1 - \gamma} \right]$.

We first prove Theorem~\ref{thm:min_q_underestimates}, which shows that policy evaluation using a simplified version of CQL (Equation~\ref{eqn:objective_1}) results in a point-wise lower-bound on the Q-function. 

\textbf{Proof of Theorem~\ref{thm:min_q_underestimates}.} In order to start, we first note that the form of the resulting Q-function iterate, $\hat{Q}^k$, in the setting without function approximation. By setting the derivative of Equation~\ref{eqn:objective_1} to 0, we obtain the following expression for $\hat{Q}^{k+1}$ in terms of $\hat{Q}^k$,
\begin{equation}
    \forall~\bs, \ba \in \mathcal{D}, k, ~~ \hat{Q}^{k+1}(\bs, \ba) = \hat{\bellman}^\pi \hat{Q}^k(\bs, \ba) - \alpha \frac{\mu(\ba|\bs)}{\hatbehavior(\ba|\bs)}.
    \label{eqn:q_expression_objective1}
\end{equation}
Now, since, $\mu(\ba|\bs) > 0, \alpha > 0, \hatbehavior(\ba|\bs) > 0$, we observe that at each iteration we underestimate the next Q-value iterate, i.e. $\hat{Q}^{k+1} \leq \hat{\bellman}^\policy \hat{Q}^k$.

\textbf{Accounting for sampling error.} Note that so far we have only shown that the Q-values are upper-bounded by the the ``empirical Bellman targets'' given by, $\hat{\bellman}^\policy \hat{Q}^k$. In order to relate $\hat{Q}^k$ to the true Q-value iterate, $Q^k$, we need to relate the empirical Bellman operator, $\hat{\bellman}^\policy$ to the actual Bellman operator, $\bellman^\policy$. In Appendix~\ref{app:handling_unobserved_actions}, we show that if the reward function $r(\bs, \ba)$ and the transition function, $\transitions(\bs'|\bs, \ba)$ satisfy ``concentration'' properties, meaning that the difference between the observed reward sample, $r$ ($\bs, \ba, r, \bs') \in \mathcal{D}$) and the actual reward function $r(\bs, \ba)$ (and analogously for the transition matrix) is bounded with high probability, then overestimation due to the empirical Backup operator is bounded. Formally, with high probability (w.h.p.) $\geq 1 - \delta$, $\delta \in (0, 1)$, 
\begin{equation*}
    \forall Q, \bs, \ba \in \mathcal{D},~~ \left\vert \hat{\bellman}^\policy Q(\bs, \ba) - \bellman^\policy Q(\bs, \ba) \right\vert \leq \frac{C_{r, T, \delta} R_{\max}}{(1 - \gamma) \sqrt{|\mathcal{D}(\bs, \ba)|}}.
\end{equation*}
Hence, the following can be obtained, w.h.p.:
\begin{align}
\label{eqn:pac_bound_q_value}
    \hat{Q}^{k+1}(\bs, \ba) = \bellman^\policy \hat{Q}^k(\bs, \ba) \leq \bellman^\policy \hat{Q}^k(\bs, \ba) - \alpha \frac{\mu(\ba|\bs)}{\hatbehavior(\ba|\bs)} + \frac{C_{r, T, \delta} R_{\max}}{(1 - \gamma) \sqrt{|\mathcal{D}(\bs, \ba)|}}.
\end{align}

Now we need to reason about the fixed point of the update procedure in Equation~\ref{eqn:q_expression_objective1}. The fixed point of Equation~\ref{eqn:q_expression_objective1} is given by:
\begin{multline*}
    \hat{Q}^\policy \leq \bellman^\policy \hat{Q}^\policy - \alpha \frac{\mu(\ba|\bs)}{\hatbehavior(\ba|\bs)} + \frac{C_{r, T, \delta} R_{\max}}{(1 - \gamma) \sqrt{|\mathcal{D}(\bs, \ba)|}} \implies \hat{Q}^\pi \leq (I - \gamma P^\pi)^{-1} \left[ R  - \alpha \frac{\mu}{\hatbehavior} + \frac{C_{r, T, \delta} R_{\max}}{1 - \gamma) \sqrt{|\mathcal{D}}}\right]\\
    \hat{Q}^\policy(\bs, \ba) \leq Q^\policy(\bs, \ba) - \alpha \left[ \left(I - \gamma P^\pi \right)^{-1} \left[\frac{\mu}{\hatbehavior} \right] \right](\bs, \ba) + \left[(I - \gamma P^\policy)^{-1} \frac{C_{r, T, \delta} R_{\max}}{(1 - \gamma) \sqrt{|\mathcal{D}|}} \right](\bs, \ba),
\end{multline*}
thus proving the relationship in Theorem~\ref{thm:min_q_underestimates}.

In order to guarantee a lower bound, $\alpha$ can be chosen to cancel any potential overestimation incurred by $\frac{C_{r, T, \delta} R_{\max}}{(1 - \gamma)\sqrt{|\mathcal{D}|}}$. Note that this  choice works, since $(I - \gamma P^\pi)^{-1}$ is a matrix with all non-negative entries. The choice of $\alpha$ that guarantees a lower bound is then given by:
\begin{align*}
    \alpha& ~ \cdot \min_{\bs, \ba} \left[\frac{\mu(\ba|\bs)}{\hatbehavior(\ba|\bs)} \right] \geq \max_{\bs, \ba} \frac{C_{r, T, \delta} R_{\max}}{(1 - \gamma) \sqrt{|\mathcal{D}(\bs, \ba)|}}\\
    \implies \alpha&~ \geq \max_{\bs, \ba} \frac{C_{r, T, \delta} R_{\max}}{(1 - \gamma) \sqrt{|\mathcal{D}(\bs, \ba)|}} \cdot \max_{\bs, \ba} \left[\frac{\mu(\ba|\bs)}{\hatbehavior(\ba|\bs)} \right]^{-1}.
\end{align*}
Note that the theoretically minimum possible value of $\alpha$ decreases as more samples are observed, i.e., when $|\mathcal{D}(\bs, \ba)|$ is large. Also, note that since, $\frac{C_{r, T, \delta} R_{\max}}{(1 - \gamma0 \sqrt{|\mathcal{D}|}} \approx 0$, when $\hat{\bellman}^\policy = \bellman^\policy$, any $\alpha \geq 0$ guarantees a lower bound. And so choosing a value of $\alpha = 0$ is sufficient in this case.

Next, we prove Theorem~\ref{thm:cql_underestimation} that shows that the additional term that maximizes the expected Q-value under the dataset distribution, $\mathbb{D}(\bs, \ba)$, (or $d^\behavior(\bs) \behavior(\ba|\bs)$, in the absence of sampling error), results in a lower-bound on only the expected value of the policy at a state, and not a pointwise lower-bound on Q-values at all actions.

\textbf{Proof of Theorem~\ref{thm:cql_underestimates}.} We first prove this theorem in the absence of sampling error, and then incorporate sampling error at the end, using a technique similar to the previous proof. In the tabular setting, we can set the derivative of the modified objective in Equation~\ref{eqn:modified_policy_eval}, and compute the Q-function update induced in the exact, tabular setting (this assumes $\hat{\bellman}^\policy = \bellman^\policy)$ and $\behavior(\ba|\bs) = \hatbehavior(\ba|\bs)$).
\begin{equation}
    \forall ~\bs, \ba, k ~~ \hat{Q}^{k+1} (\bs, \ba) = \bellman^\policy \hat{Q}^k(\bs, \ba) - \alpha \left[\frac{\mu(\ba|\bs)}{\behavior(\ba|\bs)} - 1 \right].
    \label{eqn:q_function_modified_eval}
\end{equation}
Note that for state-action pairs, $(\bs, \ba)$, such that, $\mu(\ba|\bs) < \behavior(\ba|\bs)$, we are infact adding a positive quantity, $1 - \frac{\mu(\ba|\bs)}{\behavior(\ba|\bs)}$, to the Q-function obtained, and this we cannot guarantee a point-wise lower bound, i.e. $\exists~ \bs, \ba, \text{~~s.t.}~~ \hat{Q}^{k+1}(\bs, \ba) \geq Q^{k+1}(\bs, \ba)$. To formally prove this, we can construct a counter-example three-state, two-action MDP, and choose a specific behavior policy $\policy(\ba|\bs)$, such that this is indeed the case.

The value of the policy, on the other hand, $\hat{V}^{k+1}$ is underestimated, since:
\begin{equation}
    \hat{V}^{k+1}(\bs) := \E_{\ba \sim \policy(\ba|\bs)} \left[ \hat{Q}^{k+1}(\bs, \ba) \right] = \bellman^\policy \hat{V}^k (\bs) - \alpha \E_{\ba \sim \policy(\ba|\bs)}\left[\frac{\mu(\ba|\bs)}{\behavior(\ba|\bs)} - 1 \right].
    \label{eqn:value_recursion}
\end{equation}
and we can show that $D_{\text{CQL}}(\bs): = \sum_{\ba} \policy(\ba|\bs) \left[\frac{\mu(\ba|\bs)}{\behavior(\ba|\bs)} - 1 \right]$ is always positive, when $\policy(\ba|\bs) = \mu(\ba|\bs)$. To note this, we present the following derivation:
\begin{align*}
    D_{\text{CQL}}(\bs) &:=~ \sum_{\ba} \policy(\ba|\bs) \left[\frac{\mu(\ba|\bs)}{\behavior(\ba|\bs)} - 1 \right]\\
    &= \sum_{\ba} (\policy(\ba|\bs) - \behavior(\ba|\bs) + \behavior(\ba|\bs)) \left[\frac{\mu(\ba|\bs)}{\behavior(\ba|\bs)} - 1 \right]\\
    &= \sum_{\ba} (\policy(\ba|\bs) - \behavior(\ba|\bs)) \left[ \frac{\policy(\ba|\bs) - \behavior(\ba|\bs)}{\behavior(\ba|\bs} \right] + \sum_{\ba} \behavior(\ba|\bs) \left[\frac{\mu(\ba|\bs)}{\behavior(\ba|\bs)} - 1 \right]\\
    &= \sum_{\ba} \underbrace{\left[ \frac{\left(\policy(\ba|\bs) - \behavior(\ba|\bs) \right)^2}{\behavior(\ba|\bs)} \right]}_{\geq 0}~ +~ 0 \text{~~~since,~} \sum_{\ba} \policy(\ba|\bs) = \sum_{\ba} \behavior(\ba|\bs) = 1.
\end{align*}
Note that the marked term, is positive since both the numerator and denominator are positive, and this implies that $D_\text{CQL}(\bs) \geq 0$. Also, note that $D_\text{CQL}(\bs) = 0$, iff $\policy(\ba|\bs) = \behavior(\ba|\bs)$. This implies that each value iterate incurs some underestimation, $\hat{V}^{k+1}(\bs) \leq \bellman^\policy \hat{V}^k (\bs)$.

Now, we can compute the fixed point of the recursion in Equation~\ref{eqn:value_recursion}, and this gives us the following estimated policy value:
\begin{equation*}
    \hat{V}^\policy(\bs) = V^\policy(\bs) - \alpha \left[ \underbrace{(I - \gamma P^\policy)^{-1}}_{\text{non-negative entries}}
    \underbrace{\E_{\policy}\left[\frac{\policy}{\behavior} - 1 \right]}_{\geq 0} \right](\bs),
\end{equation*}
thus showing that in the absence of sampling error, Theorem~\ref{thm:cql_underestimates} gives a lower bound. It is straightforward to note that this expression is tighter than the expression for policy value in Proposition~\ref{thm:cql_underestimates}, since, we explicitly subtract $1$ in the expression of Q-values (in the exact case) from the previous proof.

\textbf{Incorporating sampling error.} To extend this result to the setting with sampling error, similar to the previous result, the maximal overestimation at each iteration $k$, is bounded by $\frac{C_{r, T, \delta} R_{\max}}{1 - \gamma}$.
The resulting value-function satisfies (w.h.p.), $\forall \bs \in \mathcal{D}$, 
\begin{equation*}
   \hat{V}^\policy(\bs) \leq V^\policy(\bs) - \alpha \left[\left(I - \gamma P^\pi \right)^{-1} \E_{\policy}\left[\frac{\policy}{\hatbehavior} - 1 \right] \right](\bs) + \left[ (I - \gamma P^\pi)^{-1} \frac{C_{r, T, \delta} R_{\max}}{(1- \gamma) \sqrt{|\mathcal{D}|}}\right](\bs)
\end{equation*}
thus proving the theorem statement. In this case, the choice of $\alpha$, that prevents overestimation w.h.p. is given by:
\begin{equation*}
    \alpha \geq \max_{\bs, \ba \in \mathcal{D}} \frac{C_{r, T} R_{\max}}{(1 - \gamma) \sqrt{|\mathcal{D}(\bs, \ba)|}} \cdot \max_{\bs \in \mathcal{D}} \left[\sum_{\ba} \policy(\ba|\bs) \left(\frac{\policy(\ba|\bs)}{\hatbehavior(\ba|\bs))} - 1\right)\right]^{-1}.
\end{equation*}
Similar to Theorem~\ref{thm:min_q_underestimates}, note that the theoretically acceptable value of $\alpha$ decays as the number of occurrences of a state action pair in the dataset increases.
Next we provide a proof for Theorem~\ref{thm:cql_underestimation}.

\textbf{Proof of Theorem~\ref{thm:cql_underestimation}.} In order to prove this theorem, we compute the difference induced in the policy value, $\hat{V}^{k+1}$, derived from the Q-value iterate, $\hat{Q}^{k+1}$, with respect to the previous iterate $\bellman^\policy \hat{Q}^{k}$. If this difference is negative at each iteration, then the resulting Q-values are guaranteed to lower bound the true policy value.

\begin{align*}
    \E_{\hat{\policy}^{k+1}(\ba|\bs)}[\hat{Q}^{k+1}(\bs, \ba)] &=  \E_{\hat{\policy}^{k+1}(\ba|\bs)}\left[ \bellman^\policy \hat{Q}^k (\bs, \ba) \right] - \E_{\hat{\policy}^{k+1}(\ba|\bs)}\left[ \frac{\pi_{\hat{Q}^k}(\ba | \bs)}{\hatbehavior(\ba|\bs)} -1 \right]\\
    &= \E_{\hat{\policy}^{k+1}(\ba|\bs)}\left[ \bellman^\policy \hat{Q}^k (\bs, \ba) \right] - \underbrace{\E_{\policy_{\hat{Q}^k}(\ba | \bs)}\left[ \frac{\policy_{\hat{Q}^k}(\ba | \bs)}{\hatbehavior(\ba|\bs)} -1 \right]}_{\text{underestimation, ~~(a)}}\\
    &~~~~~~~~~~~~~~~~~~~~~~~~~~~~~+ \sum_{\ba} \underbrace{\left(\policy_{\hat{Q}^k}(\ba | \bs) - \hat{\policy}^{k+1}(\ba|\bs) \right)}_{\text{(b),}~~{\leq \mathrm{D_{TV}}(\policy_{\hat{Q}^k}, \hat{\policy}^{k+1})}} \frac{\policy_{\hat{Q}^k}(\ba | \bs)}{\hatbehavior(\ba|\bs)}
\end{align*}
If (a) has a larger magnitude than (b), then the learned Q-value induces an underestimation in an iteration $k+1$, and hence, by a recursive argument, the learned Q-value underestimates the optimal Q-value. We note that by upper bounding term (b) by $\mathrm{D_{TV}}(\policy_{\hat{Q}^k}, \hat{\policy}^{k+1}) \cdot \max_{\ba} \frac{\policy_{\hat{Q}^k}(\ba | \bs)}{\hatbehavior(\ba|\bs)}$, and writing out (a) > upper-bound on (b), we obtain the desired result.

Finally, we show that under specific choices of $\alpha_1, \cdots, \alpha_k$, the CQL backup is gap-expanding by providing a proof for Theorem~\ref{thm:gap_amplify}

\textbf{Proof of Theorem~\ref{thm:gap_amplify} (CQL is gap-expanding).} For this theorem, we again first present the proof in the absence of sampling error, and then incorporate sampling error into the choice of $\alpha$. We follow the strategy of observing the Q-value update in one iteration. Recall that the expression for the Q-value iterate at iteration $k$ is given by:
\begin{equation*}
    \hat{Q}^{k+1}(\bs, \ba) = \bellman^{\policy^k} \hat{Q}^{k}(\bs, \ba) - \alpha_k \frac{\mu_k(\ba|\bs) - \behavior(\ba|\bs)}{\behavior(\ba|\bs)}.
\end{equation*}
Now, the value of the policy $\mu_{k}(\ba|\bs)$ under $\hat{Q}^{k+1}$ is given by:
\begin{equation*}
    \E_{\ba \sim \mu_k(\ba|\bs)}[\hat{Q}^{k+1}(\bs, \ba)] = \E_{\ba \sim \mu_k(\ba|\bs)}[\bellman^{\policy^k} \hat{Q}^k (\bs, \ba)] - \alpha_k \underbrace{\mu_k^T \left(\frac{\mu_k(\ba|\bs) - \behavior(\ba|\bs)}{\behavior(\ba|\bs)} \right)}_{:= \hat{\Delta}^k, \geq 0, \text{~~by proof of Theorem~\ref{thm:cql_underestimates}.}} 
\end{equation*}
Now, we also note that the expected amount of extra underestimation introduced at iteration $k$ under action sampled from the behavior policy $\behavior(\ba|\bs)$ is 0, as,
\begin{equation*}
    \E_{\ba \sim \behavior(\ba|\bs)}[\hat{Q}^{k+1}(\bs, \ba)] = \E_{\ba \sim \behavior(\ba|\bs)}[\bellman^{\policy^k} \hat{Q}^k (\bs, \ba)] - \alpha_k \underbrace{\behavior^T \left(\frac{\mu_k(\ba|\bs) - \behavior(\ba|\bs)}{\behavior(\ba|\bs)} \right)}_{= 0}. 
\end{equation*}
where the marked quantity is equal to 0 since it is equal since $\behavior(\ba|\bs)$ in the numerator cancels with the denominator, and the remaining quantity is a sum of difference between two density functions, $\sum_{\ba} \mu_k(\ba|\bs) - \behavior(\ba|\bs)$, which is equal to 0. Thus, we have shown that,
\begin{equation*}
    \E_{\behavior(\ba|\bs)}[\hat{Q}^{k+1}(\bs, \ba)] - \E_{\mu_k(\ba|\bs)}[\hat{Q}^{k+1}(\bs, \ba)] = \E_{\behavior(\ba|\bs)}[\bellman^{\policy^k} \hat{Q}^k (\bs, \ba)] - \E_{\mu_k(\ba|\bs)}[\bellman^{\policy^k} \hat{Q}^k (\bs, \ba)] - \alpha_k \hat{\Delta}^k.
\end{equation*}
Now subtracting the difference, $\E_{\behavior(\ba|\bs)}[Q^{k+1}(\bs, \ba)] - \E_{\mu_k(\ba|\bs)}[Q^{k+1}(\bs, \ba)]$, computed under the tabular Q-function iterate, $Q^{k+1}$, from the previous equation, we obtain that
\begin{multline*}
    \E_{\ba \sim \behavior(\ba|\bs)}[\hat{Q}^{k+1}(\bs, \ba)] - \E_{\behavior(\ba|\bs)}[{Q}^{k+1}(\bs, \ba)] = \E_{\mu_k(\ba|\bs)}[\hat{Q}^{k+1}(\bs, \ba)] - \E_{\mu_k(\ba|\bs)}[{Q}^{k+1}(\bs, \ba)]\\
    + (\mu_k(\ba|\bs) - \behavior(\ba|\bs))^T \underbrace{\left[ \bellman^{\policy^k} \left( \hat{Q}^k - Q^k \right)(\bs, \cdot) \right]}_{(a)}- \alpha_k \hat{\Delta}^k.
\end{multline*}

Now, by choosing $\alpha_k$, such that any positive bias introduced by the quantity $(\mu_k(\ba|\bs) - \behavior(\ba|\bs))^T (a)$ is cancelled out, we obtain the following gap-expanding relationship:
\begin{equation*}
    \E_{\ba \sim \behavior(\ba|\bs)}[\hat{Q}^{k+1}(\bs, \ba)] - \E_{\behavior(\ba|\bs)}[{Q}^{k+1}(\bs, \ba)] > \E_{\mu_k(\ba|\bs)}[\hat{Q}^{k+1}(\bs, \ba)] - \E_{\mu_k(\ba|\bs)}[{Q}^{k+1}(\bs, \ba)]
\end{equation*}
for, $\alpha_k$ satisfying, 
\begin{equation*}
    \alpha_k > \max \left(\frac{(\behavior(\ba|\bs) - \mu_k(\ba|\bs))^T \left[ \bellman^{\policy^k} \left( \hat{Q}^k - Q^k \right)(\bs, \cdot) \right]}{\hat{\Delta}^k}, 0\right),
\end{equation*}
thus proving the desired result.

To avoid the dependency on the true Q-value iterate, ${Q}^k$, we can upper-bound $Q^k$ by $\frac{R_{\max}}{1 - \gamma}$, and upper-bound $(\behavior(\ba|\bs) - \mu_k(\ba|\bs))^T \bellman^{\policy^k} Q^k (\bs, \cdot)$ by $D_{\mathrm{TV}}(\behavior, \mu_k) \cdot \frac{R_{\max}}{1 - \gamma}$, and use this in the expression for $\alpha_k$. While this bound may be loose, it still guarantees the gap-expanding property, and we indeed empirically show the existence of this property in practice in Appendix~\ref{app:gap_amplify}. 

To incorporate sampling error, we can follow a similar strategy as previous proofs: the worst case overestimation due to sampling error is given by $\frac{C_{r, T, \delta} R_{\max}}{1 - \gamma}$. In this case, we note that, w.h.p.,
\begin{equation*}
    \left\vert \hat{\bellman}^{\policy^k} \left( \hat{Q}^k - Q^k \right) - \bellman^{\policy^k} \left( \hat{Q}^k - Q^k \right) \right\vert \leq \frac{2 \cdot C_{r, T, \delta} R_{\max}}{1 - \gamma}. 
\end{equation*}
Hence, the presence of sampling error adds $D_{\mathrm{TV}}(\hat{\behavior}, \mu_k) \cdot \frac{2 \cdot C_{r, T, \delta} R_{\max}}{1 - \gamma}$ to the value of $\alpha_k$, giving rise to the following, sufficient condition on $\alpha_k$ for the gap-expanding property:
\begin{equation*}
     \alpha_k > \max \left(\frac{(\behavior(\ba|\bs) - \mu_k(\ba|\bs))^T \left[ \bellman^{\policy^k} \left( \hat{Q}^k - Q^k \right)(\bs, \cdot) \right]}{\hat{\Delta}^k} + D_{\mathrm{TV}}(\hat{\behavior}, \mu_k) \cdot \frac{2 \cdot C_{r, T, \delta} R_{\max}}{1 - \gamma}, 0\right),
\end{equation*}
concluding the proof of this theorem.

\section{Additional Theoretical Analysis}
\label{app:additional_theory}
In this section, we present a theoretical analysis of additional properties of CQL. For ease of presentation, we state and prove theorems in Appendices~\ref{app:cql_linear_non_linear} and \ref{app:maximizing_distributions} in the absence of sampling error, but as discussed extensively in Appendix~\ref{app:missing_proofs}, we can extend each of these results by adding extra terms induced due to sampling error. 

\subsection{CQL with Linear and Non-Linear Function Approximation}
\label{app:cql_linear_non_linear}
\begin{theorem}
\label{thm:policy_eval_func_approx}
Assume that the Q-function is represented as a linear function of given state-action feature vectors $\Qfeat$, i.e., $Q(s, a) = w^T \Qfeat(s, a)$. Let $D = \text{diag}\left(d^\behavior(\bs) \behavior(\ba|\bs)\right)$ denote the diagonal matrix with data density, and assume that $\Qfeat^T D \Qfeat$ is invertible. Then, the expected value of the policy under Q-value from Eqn~\ref{eqn:modified_policy_eval} at iteration $k+1$, $\E_{d^\behavior(\ba)}[\hat{V}^{k+1}(\bs)] = \E_{d^\behavior(\bs), \policy(\ba|\bs)}[\hat{Q}^{k+1}(\bs, \ba)]$, lower-bounds the corresponding tabular value, $\E_{d^\behavior(\bs)}[V^{k+1}(\bs)] =\E_{d^\behavior(\bs), \policy(\ba|\bs)}[Q^{k+1}(\bs, \ba)]$, if
\begin{equation*}
\small{\alpha_k \geq \max \left(\frac{D^T\left[\Qfeat \left(\Qfeat^T D \Qfeat \right)^{-1} \Qfeat^T - I \right]\left((\bellman^\policy \hat{Q}^k)(\bs, \ba)\right)}{D^T\left[ \Qfeat \left(\Qfeat^T D \Qfeat \right)^{-1} \Qfeat^T \right] \left(D \left[\frac{\policy(\ba|\bs) - \behavior(\ba|\bs)}{\behavior(\ba|\bs)} \right] \right)},~ 0 \right).}
\end{equation*}
\end{theorem}
The choice of $\alpha_k$ in Theorem~\ref{thm:policy_eval_func_approx} intuitively amounts to compensating for overestimation in value induced if the true value function cannot be represented in the chosen linear function class (numerator), by the potential decrease in value due to the CQL regularizer (denominator). This implies that if the actual value function can be represented in the linear function class, such that the numerator can be made $0$, then \textbf{any} $\alpha > 0$ is sufficient to obtain a lower bound. We now prove the theorem.

\begin{proof} 
In order to extend the result of Theorem~\ref{thm:cql_underestimates} to account for function approximation, we follow the similar recipe as before. We obtain the optimal solution to the optimization problem below in the family of linearly expressible Q-functions, i.e. $\mathcal{Q} := \{\Qfeat w | w \in \mathbb{R}^{dim(\Qfeat)} \}$. 
\begin{equation*}
    \min_{Q \in \mathcal{Q}}~~ \alpha_k \cdot \left(\E_{d^\behavior(\bs),\mu(\ba|\bs)}\left[Q(\bs, \ba)\right] - {\E_{d^\behavior(\bs),\behavior(\ba|\bs)}\left[Q(\bs, \ba)\right]} \right) + \frac{1}{2}~ \E_{\mathcal{D}}\left[\left(Q(\bs, \ba) - \bellman^\policy \hat{Q}^{k} (\bs, \ba) \right)^2 \right].
\end{equation*}
By substituting $Q(\bs, \ba) = w^T \Qfeat(\bs, \ba)$, and setting the derivative with respect to $w$ to be 0, we obtain,
\begin{align*}
    \alpha \sum_{\bs, \ba} d^\behavior(\bs) \cdot \left( \mu(\ba|\bs) - \behavior(\ba|\bs) \right) \Qfeat(\bs, \ba) + \sum_{\bs, \ba} d^\behavior(\bs) \behavior(\ba|\bs) \left( Q(\bs, \ba) - \bellman^\pi \hat{Q}^k(\bs, \ba) \right) \Qfeat(\bs, \ba) = 0.
\end{align*}
By re-arranging terms, and converting it to vector notation, defining $D = \text{diag}(d^\behavior(\bs)\behavior(\bs))$, and referring to the parameter $w$ at the k-iteration as $w^k$ we obtain:
\begin{align*}
    \left(\Qfeat^T D \Qfeat \right) w^{k+1} = \underbrace{\Qfeat^T D \left( \bellman^\policy \hat{Q}^k \right)}_{\text{LSTD iterate}} - \underbrace{\alpha_k \Qfeat^T \text{diag}\left[d^\behavior(\bs) \cdot(\mu(\ba|\bs) - \behavior(\ba|\bs)) \right]}_{\text{underestimation}}.
\end{align*}
Now, our task is to show that the term labelled as ``underestimation'' is indeed negative in expectation under $\mu(\ba|\bs)$ (This is analogous to our result in the tabular setting that shows underestimated values). In order to show this, we write out the expression for the value, under the linear weights $w^{k+1}$ at state $\bs$, after substituting $\mu = \policy$,
\begin{align}
    \hat{V}^{k+1}(\bs) &:= \policy(\ba|\bs)^T \Qfeat w^{k+1}\\
    = \policy(&\ba|\bs)^T \underbrace{\Qfeat \left(\Qfeat^T D \Qfeat \right)^{-1} \Qfeat^T D \left(\bellman^\policy \hat{Q}^k \right)}_{\text{value under LSTD-Q~\citep{lagoudakis2003least}}} - \alpha_k \policy(\ba|\bs)^T \Qfeat \left(\Qfeat^T D \Qfeat \right)^{-1} \Qfeat^T D \left[\frac{\policy(\ba|\bs) - \behavior(\ba|\bs)}{\behavior(\ba|\bs)} \right].
    \label{expr:lstdq_value}
\end{align}
Now, we need to reason about the penalty term. Defining, $P_\Qfeat := \Qfeat \left(\Qfeat^T D \Qfeat \right)^{-1} \Qfeat^T D$ as the projection matrix onto the subspace of features $\Qfeat$, we need to show that the product that appears as a penalty is positive: $\policy(\ba|\bs)^T P_\Qfeat \left[\frac{\policy(\ba|\bs) - \behavior(\ba|\bs)}{\behavior(\ba|\bs)} \right] \geq 0$. In order to show this, we compute minimum value of this product optimizing over $\policy$. If the minimum value is $0$, then we are done.

Let's define $f(\policy) = \policy(\ba|\bs)^T P_\Qfeat \left[\frac{\policy(\ba|\bs) - \behavior(\ba|\bs)}{\behavior(\ba|\bs)} \right]$, our goal is to solve for $\min_{\policy} f(\policy)$. Setting the derivative of $f(\policy)$ with respect to $\policy$ to be equal to 0, we obtain (including Lagrange multiplier $\eta$ that guarantees $\sum_{\ba} \policy(\ba|\bs) = 1$,
\begin{equation*}
    \left( P_\Qfeat + P_\Qfeat^T \right) \left[\frac{\policy(\ba|\bs)}{\behavior(\ba|\bs)}\right] = P_\Qfeat \Vec{1} + \eta \Vec{1}.
\end{equation*}
By solving for $\eta$ (using the condition that a density function sums to 1), we obtain that the minimum value of $f(\policy)$ occurs at a $\policy^*(\ba|\bs)$, which satisfies the following condition,
\begin{equation*}
    \left(P_\Qfeat + P_\Qfeat^T \right) \left[\frac{\policy^*(\ba|\bs)}{\behavior(\ba|\bs)}\right] = \left( P_\Qfeat + P_\Qfeat^T \right) \Vec{1}.
\end{equation*}
Using this relation to compute $f$, we obtain, $f(\policy^*) = 0$, indicating that the minimum value of $0$ occurs when the projected density ratio matches under the matrix $(P_\Qfeat + P_\Qfeat^T)$ is equal to the projection of a vector of ones, $\Vec{1}$. Thus, \begin{equation*}
    \forall~ \policy(\ba|\bs),~ f(\policy) = \policy(\ba|\bs)^T P_\Qfeat \left[\frac{\policy(\ba|\bs) - \behavior(\ba|\bs)}{\behavior(\ba|\bs)} \right] \geq 0.
\end{equation*}

This means that: $\forall ~\bs, \hat{V}^{k+1}(\bs) \leq \hat{V}^{k+1}_{\text{LSTD-Q}}(\bs)$ given identical previous $\hat{Q}^{k}$ values. This result indicates, that if $\alpha_k \geq 0$, the resulting CQL value estimate with linear function approximation is guaranteed to lower-bound the value estimate obtained from a least-squares temporal difference Q-learning algorithm (which only minimizes Bellman error assuming a linear Q-function parameterization), such as LSTD-Q~\citep{lagoudakis2003least}, since at each iteration, CQL induces a lower-bound with respect to the previous value iterate, whereas this underestimation is absent in LSTD-Q, and an inductive argument is applicable.

So far, we have only shown that the learned value iterate, $\hat{V}^{k+1}(\bs)$ lower-bounds the value iterate obtained from LSTD-Q, $\forall ~\bs, \hat{V}^{k+1}(\bs) \leq \hat{V}^{k+1}_{\text{LSTD-Q}}(\bs)$. But, our final aim is to prove a stronger result, that the learned value iterate, $\hat{V}^{k+1}$, lower bounds the exact tabular value function iterate, ${V}^{k+1}$, at each iteration. The reason why our current result does not guarantee this is because function approximation may induce overestimation error in the linear approximation of the Q-function.

In order to account for this change, we make a simple change: we choose $\alpha_k$ such that the resulting penalty nullifes the effect of any over-estimation caused due to the inability to fit the true value function iterate in the linear function class parameterized by $\Qfeat$. Formally, this means:
\begin{align*}
    \E_{d^\behavior(\bs)}\left[\hat{V}^{k+1}(\bs)\right] &\leq \E_{d^\behavior(\bs)}\left[\hat{V}^{k+1}_{\text{LSTD-Q}}(\bs)\right] - \alpha_k \E_{d^\behavior(\bs)}[f(\policy(\ba|\bs))] \\
    & \leq \E_{d^\behavior(\bs)}\left[V^{k+1}(\bs)\right] - \underbrace{\E_{d^\behavior(\bs)}\left[\hat{V}^{k+1}_{\text{LSTD-Q}}(\bs) - V^{k+1}(\bs) \right] - \alpha_k \E_{d^\behavior(\bs)}[f(\policy(\ba|\bs))]}_{\text{choose~} \alpha_k \text{~to make this negative}} \\
    & \leq \E_{d^\behavior(\bs)}\left[V^{k+1}(\bs)\right]
\end{align*}
And the choice of $\alpha_k$ in that case is given by:
\begin{align*}
    \alpha_k \geq& ~ \max \left(\frac{\E_{d^\behavior(\bs)}\left[\hat{V}^{k+1}_{\text{LSTD-Q}}(\bs) - V^{k+1}(\bs) \right]}{\E_{d^\behavior(\bs)}[f(\policy(\ba|\bs))]}, 0 \right)\\
    \geq& ~ \max \left(\frac{D^T\left[\Qfeat \left(\Qfeat^T D \Qfeat \right)^{-1} \Qfeat^T - I \right]\left((\bellman^\policy \hat{Q}^k)(\bs, \ba)\right)}{D^T\left[ \Qfeat \left(\Qfeat^T D \Qfeat \right)^{-1} \Qfeat^T \right] \left(D \left[\frac{\policy(\ba|\bs) - \behavior(\ba|\bs)}{\behavior(\ba|\bs)} \right] \right)},~ 0 \right).
\end{align*}
Finally, we note that since this choice of $\alpha_k$ induces under-estimation in the next iterate, $\hat{V}^{k+1}$ with respect to the previous iterate, $\hat{V}^k$, for all $k \in \mathbb{N}$, by induction, we can claim that this choice of $\alpha_1, \cdots, \alpha_k$ is sufficient to make $\hat{V}^{k+1}$ lower-bound the tabular, exact value-function iterate. $V^{k+1}$, for all $k$, thus completing our proof.
\end{proof}

We can generalize Theorem~\ref{thm:policy_eval_func_approx} to non-linear function approximation, such as neural networks, {under the standard NTK framework~\citep{ntk}}, assuming that each iteration $k$ is performed by a single step of gradient descent on Equation~\ref{eqn:modified_policy_eval}, rather than a complete minimization of this objective. As we show in Theorem~\ref{corr:nonlinear_ntk}, CQL learns lower bounds in this case for an appropriate choice of $\alpha_k$. We will also empirically show in Appendix~\ref{app:additional_results} that CQL can learn effective conservative Q-functions with multilayer neural networks.

\begin{theorem}[Extension to non-linear function approximation]
\label{corr:nonlinear_ntk}
Assume that the Q-function is represented by a general non-linear function approximator parameterized by $\theta$, $Q_\theta(\bs, \ba)$. let $D = \text{diag}(d^\behavior(\bs) \behavior(\ba|\bs))$ denote the matrix with the data density on the diagonal, and assume that $\nabla_\theta Q_\theta^T D \nabla_\theta Q_\theta$ is invertible. Then, the expected value of the policy under the Q-function obtained by taking a gradient step on Equation~\ref{eqn:modified_policy_eval}, at iteration $k+1$ lower-bounds the corresponding tabular function iterate if: 
\end{theorem}

\begin{proof} Our proof strategy is to reduce the non-linear optimization problem into a linear one, with features $\Qfeat$ (in Theorem~\ref{thm:policy_eval_func_approx}) replaced with features given by the gradient of the current Q-function $\hat{Q}^{k}_\theta$ with respect to parameters $\theta$, i.e. $\nabla_\theta \hat{Q}^k$. To see, this we start by writing down the expression for $\hat{Q}^{k+1}_\theta$ obtained via one step of gradient descent with step size $\eta$, on the objective in Equation~\ref{eqn:modified_policy_eval}.
\begin{align*}
    \theta^{k+1} = \theta^k &- \eta \alpha_k \left( \E_{d^\behavior(\bs), \mu(\ba|\bs)}\left[\nabla_\theta \hat{Q}^{k}(\bs, \ba)\right] - \E_{d^\behavior(\bs), \behavior(\ba|\bs)}\left[\nabla_\theta \hat{Q}^{k}(\bs, \ba)\right] \right) \\ 
    &- \eta \E_{d^\behavior(\bs), \behavior(\ba|\bs)}\left[ \left(\hat{Q}^k - \bellman^\policy \hat{Q}^{k} \right)(\bs, \ba) \cdot \nabla_\theta \hat{Q}^{k}(\bs, \ba) \right].
\end{align*}
Using the above equation and making an approximation linearization assumption on the non-linear Q-function, for small learning rates $\eta << 1$, as has been commonly used by prior works on the neural tangent kernel (NTK) in deep learning theory~\citep{ntk} in order to explain neural network learning dynamics in the infinite-width limit, we can write out the expression for the next Q-function iterate, $\hat{Q}^{k+1}_\theta$ in terms of $\hat{Q}^k_\theta$ as~\citep{achiam2019towards,ntk}:
\begin{align*}
    \hat{Q}^{k+1}_\theta(\bs, \ba) &\approx \hat{Q}^k_\theta(\bs, \ba) + \left(\theta^{k+1} - \theta^{k}\right)^T \nabla_\theta \hat{Q}^{k}_\theta (\bs, \ba) ~~~~ \text{(under NTK assumptions)}\\
    &= \hat{Q}^k_\theta(\bs, \ba) - \eta \alpha_k \E_{d^\behavior(\bs'), \mu(\ba'|\bs')}\left[\nabla_\theta \hat{Q}^{k}(\bs', \ba')^T \nabla_\theta \hat{Q}^{k}(\bs, \ba) \right]\\
    &~~ + \eta \alpha_k \E_{d^\behavior(\bs'), \behavior(\ba'|\bs')}\left[\nabla_\theta \hat{Q}^{k}(\bs', \ba')^T \nabla_\theta \hat{Q}^k(\bs, \ba) \right]\\
    &~~ - \eta \E_{d^\behavior(\bs'), \behavior(\ba'|\bs')}\left[ \left(\hat{Q}^k - \bellman^\policy \hat{Q}^{k} \right)(\bs', \ba') \cdot \nabla_\theta \hat{Q}^{k}(\bs', \ba')^T \nabla_\theta \hat{Q}^k(\bs, \ba) \right].
\end{align*}
To simplify presentation, we convert into matrix notation, where we define the $|\mathcal{S}||\mathcal{A}| \times |\mathcal{S}||\mathcal{A}|$ matrix, $\bM^k = \left(\nabla_\theta \hat{Q}^k\right)^T \nabla_\theta \hat{Q}^k$, as the neural tangent kernel matrix of the Q-function at iteration $k$. Then, the vectorized $\hat{Q}^{k+1}$ (with $\mu = \policy$) is given by,
\begin{align*}
    \hat{Q}^{k+1} = \hat{Q}^{k} -&~ \eta \alpha_k \bM^k D \left[ \frac{\policy(\ba|\bs) - \behavior(\ba|\bs)}{\behavior(\ba|\bs)} \right] + \eta \bM^k D \left(\bellman^\policy \hat{Q}^k - \hat{Q}^k \right).
\end{align*}
Finally, the value of the policy is given by:
\begin{equation}
    \hat{V}^{k+1} := \underbrace{\policy(\ba|\bs)^T \hat{Q}^k(\bs, \ba) + \eta \policy(\ba|\bs) \bM^k D \left(\bellman^\policy \hat{Q}^k - \hat{Q}^k \right)}_{\text{(a) unpenalized value}} - \underbrace{\eta \alpha_k \policy(\ba|\bs)^T \bM^k D \left[ \frac{\policy(\ba|\bs) - \behavior(\ba|\bs)}{\behavior(\ba|\bs)} \right]}_{\text{(b) penalty}}.
\end{equation}
Term marked (b) in the above equation is similar to the penalty term shown in Equation~\ref{expr:lstdq_value}, and by performing a similar analysis, we can show that $\text{(b)} \geq 0$. Again similar to how $\hat{V}^{k+1}_\text{LSTD-Q}$ appeared in Equation~\ref{expr:lstdq_value}, we observe that here we obtain the value function corresponding to a regular gradient-step on the Bellman error objective. 
\end{proof}

Again similar to before, term (a) can introduce overestimation relative to the tabular counterpart, starting at $\hat{Q}^k$: $Q^{k+1} = \hat{Q}^k - \eta \left(\bellman^\policy \hat{Q}^k - \hat{Q}^k \right)$, and we can choose $\alpha_k$ to compensate for this potential increase as in the proof of Theorem~\ref{thm:policy_eval_func_approx}. As the last step, we can recurse this argument to obtain our final result, for underestimation.

\subsection{Choice of Distribution to Maximize Expected Q-Value in Equation~\ref{eqn:modified_policy_eval}}
\label{app:maximizing_distributions}

In Section~\ref{sec:policy_eval}, we introduced a term that maximizes Q-values under the dataset $d^\behavior(\bs) \behavior(\ba|\bs)$ distribution when modifying Equation~\ref{eqn:objective_1} to Equation~\ref{eqn:modified_policy_eval}. Theorem~\ref{thm:cql_underestimates} indicates the ``sufficiency'' of maximizing Q-values under the dataset distribution -- this guarantees a lower-bound on value. We now investigate the neccessity of this assumption: We ask the formal question: \textbf{For which other choices of $\nu(\ba|\bs)$ for the maximization term, is the value of the policy under the learned Q-value, $\hat{Q}^{k+1}_\nu$ guaranteed to be a lower bound on the actual value of the policy?}

To recap and define notation, we restate the objective from Equation~\ref{eqn:objective_1} below.  
\begin{equation}
    \small{\hat{Q}^{k+1} \leftarrow \arg\min_{Q}~ \alpha~ \E_{\bs \sim d^\behavior(\bs), \ba \sim \mu(\ba|\bs)}\left[Q(\bs, \ba)\right] + \frac{1}{2}~ \E_{\bs, \ba, \bs' \sim \mathcal{D}}\left[\left(Q(\bs, \ba) - \bellman^\policy \hat{Q}^{k} (\bs, \ba) \right)^2 \right].} 
    \label{eqn:objective_1_restated}
\end{equation}
We define a general family of objectives from Equation~\ref{eqn:modified_policy_eval}, parameterized by a distribution $\nu$ which is chosen to maximize Q-values as shown below (CQL is a special case, with $\nu(\ba|\bs) = \behavior(\ba|\bs)$):
\begin{multline}
    \small{\hat{Q}^{k+1}_\nu \leftarrow \arg\min_{Q}~~ \alpha \cdot \left(\E_{\bs \sim d^\behavior(\bs), \ba \sim \mu(\ba|\bs)}\left[Q(\bs, \ba)\right] - \textcolor{red}{\E_{\bs \sim d^\behavior(\bs), \ba \sim \nu(\ba|\bs)}\left[Q(\bs, \ba)\right]} \right)} \\
    \small{+ \frac{1}{2}~ \E_{\bs, \ba, \bs' \sim \mathcal{D}}\left[\left(Q(\bs, \ba) - \bellman^\policy \hat{Q}^{k} (\bs, \ba) \right)^2 \right].~~~~~ (\nu\text{-CQL})}
    \label{eqn:modified_policy_eval_new}
\end{multline}

In order to answer our question, we prove the following result:
\begin{theorem}[Necessity of maximizing Q-values under $\behavior(\ba|\bs)$.] For any policy $\policy(\ba|\bs)$, any $\alpha > 0$, and for all $k > 0$, the value of the policy., $\hat{V}^{k+1}_\nu$ under Q-function iterates from $\nu-$CQL, $\hat{Q}^{k+1}_\nu(\bs, \ba)$ is guaranteed to be a lower bound on the exact value iterate, $\hat{V}^{k+1}$, only if $\nu(\ba|\bs) = \behavior(\ba|\bs)$.   
\end{theorem}

\begin{proof}
We start by noting the parametric form of the resulting tabular Q-value iterate:
\begin{equation}
    \hat{Q}^{k+1}_\nu(\bs, \ba) = \bellman^\policy \hat{Q}^{k}_\nu(\bs, \ba) - \alpha_k \frac{\mu(\ba|\bs) - \nu(\ba|\bs)}{\behavior(\ba|\bs)}.
    \label{eqn:q_value_nu}
\end{equation}
The value of the policy under this Q-value iterate, when distribution $\mu$ is chosen to be the target policy $\pi(\ba|\bs)$ i.e. $\mu(\ba|\bs) = \policy(\ba|\bs)$ is given by:
\begin{equation}
    \small{\hat{V}^{k+1}_\nu(\bs): = \E_{\ba \sim \policy(\ba|\bs)}\left[\hat{Q}^{k+1}_\nu(\bs, \ba)\right] = \E_{\ba \sim \policy(\ba|\bs)}\left[\bellman^\policy \hat{Q}^k_\nu(\bs, \ba) \right] - \alpha_k~ \policy(\ba|\bs)^T \left(\frac{\policy(\ba|\bs) - \nu(\ba|\bs)}{\behavior(\ba|\bs)}\right)}. 
\end{equation}
We are interested in conditions on $\nu(\ba|\bs)$ such that the penalty term in the above equation is positive. It is clear that choosing $\nu(\ba|\bs) = \behavior(\ba|\bs)$ returns a policy that satisfies the requirement, as shown in the proof for Theorem~\ref{thm:cql_underestimates}. In order to obtain other choices of $\nu(\ba|\bs)$ that guarantees a lower bound for all possible choices of $\policy(\ba|\bs)$, we solve the following concave-convex maxmin optimization problem, that computes a $\nu(\ba|\bs)$ for which a lower-bound is guaranteed for \textit{all} choices of $\mu(\ba|\bs)$:
\begin{align*}
    \max_{\nu(\ba|\bs)} ~\min_{\policy(\ba|\bs)}~~& ~\sum_{\ba} \policy(\ba|\bs) \cdot \left(\frac{\policy(\ba|\bs) - \nu(\ba|\bs)}{\behavior(\ba|\bs)}\right)\\
    \text{s.t.}~~& ~~ \sum_{\ba} \policy(\ba|\bs) = 1,~ \sum_{\ba} \nu(\ba|\bs) = 1,~ \nu(\ba|\bs) \geq 0,~ \policy(\ba|\bs) \geq 0. 
\end{align*}
We first solve the inner minimization over $\policy(\ba|\bs)$ for a fixed $\nu(\ba|\bs)$, by writing
out the Lagrangian and setting the gradient of the Lagrangian to be 0, we obtain:
\begin{equation*}
    \forall \ba, ~~ 2 \cdot \frac{\policy^*(\ba|\bs)}{\behavior(\ba|\bs)} - \frac{\nu(\ba|\bs)}{\behavior(\ba|\bs)} - \zeta(\ba|\bs) + \eta = 0, 
\end{equation*}
where $\zeta(\ba|\bs)$ is the Lagrange dual variable for the positivity constraints on $\policy(\ba|\bs)$, and $\eta$ is the Lagrange dual variable for the normalization constraint on $\policy$. If $\policy(\ba|\bs)$ is full support (for example, when it is chosen to be a Boltzmann policy), KKT conditions imply that, $\zeta(\ba|\bs)$ = 0, and computing $\eta$ by summing up over actions, $\ba$, the optimal choice of $\policy$ for the inner minimization is given by:
\begin{equation}
    \policy^*(\ba|\bs) = \frac{1}{2} \nu(\ba|\bs) + \frac{1}{2} \behavior(\ba|\bs).
    \label{eqn:optimal_inner}
\end{equation}
Now, plugging Equation~\ref{eqn:optimal_inner} in the original optimization problem, we obtain the following optimization over only $\nu(\ba|\bs)$:
\begin{equation}
    \max_{\nu(\ba|\bs)} \sum_{\ba} \behavior(\ba|\bs) \cdot \left( \frac{1}{2} - \frac{\nu(\ba|\bs)}{2 \behavior(\ba|\bs)} \right) \cdot \left( \frac{1}{2} + \frac{\nu(\ba|\bs)}{2 \behavior(\ba|\bs)} \right)~~~ \text{s.t.}~~ \sum_{\ba} \nu(\ba|\bs) = 1,~ \nu(\ba|\bs) \geq 0.
\end{equation}
Solving this optimization, we find that the optimal distribution, $\nu(\ba|\bs)$ is equal to $\behavior(\ba|\bs)$. and the optimal value of penalty, which is also the objective for the problem above is equal to $0$. Since we are maximizing over $\nu$, this indicates for other choices of $\nu \neq \behavior$, we can find a $\policy$ so that the penalty is negative, and hence a lower-bound is not guaranteed. Therefore, we find that with a worst case choice of $\policy(\ba|\bs)$, a lower bound can only be guaranteed only if $\nu(\ba|\bs) = \behavior(\ba|\bs)$. This justifies the necessity of $\behavior(\ba|\bs)$ for maximizing Q-values in Equation~\ref{eqn:modified_policy_eval}. The above analysis doesn't take into account the effect of function approximation or sampling error. We can, however, generalize this result to those settings, by following a similar strategy of appropriately choosing $\alpha_k$, as previously utilized in Theorem~\ref{thm:policy_eval_func_approx}.
\end{proof}

\subsection{CQL with Empirical Dataset Distributions}
\label{app:handling_unobserved_actions}

The results in Sections~\ref{sec:policy_eval} and~\ref{sec:framework} account for sampling error due to the finite size of the dataset $\mathcal{D}$. In our practical implementation as well, we optimize a sample-based version of Equation~\ref{eqn:modified_policy_eval}, as shown below:
\begin{multline}
    \small{\hat{Q}^{k+1} \leftarrow \arg\min_{Q}~~ \alpha \cdot \left(\sum_{\bs \in \mathcal{D}} \E_{\ba \sim \mu(\ba|\bs)}\left[Q(\bs, \ba)\right] - {\sum_{\bs \in \mathcal{D}} \E_{\ba \sim \behavior(\ba|\bs)}\left[Q(\bs, \ba)\right]} \right)} \\
    \small{+ \frac{1}{2 |\mathcal{D}|}~ \sum_{\bs, \ba, \bs' \in \mathcal{D}}\left[\left(Q(\bs, \ba) - \hat{\bellman}^\policy \hat{Q}^{k} (\bs, \ba) \right)^2 \right]},
    \label{eqn:sampled_policy_eval}
\end{multline}
where $\hat{\bellman}^\policy$ denotes the ``empirical'' Bellman operator computed using samples in $\mathcal{D}$ as follows:
\begin{equation}
\label{eqn:app_empirical_bellman}
    \forall \bs, \ba \in \mathcal{D}, ~~ \left(\hat{\bellman}^\policy \hat{Q}^k \right)(\bs, \ba) = r + \gamma \sum_{\bs'} \hat{\transitions}(\bs'|\bs, \ba) \E_{\ba' \sim \policy(\ba'|\bs')}\left[\hat{Q}^k(\bs' , \ba')\right],  
\end{equation}
where $r$ is the empirical average reward obtained in the dataset when executing an action $\ba$ at state $\bs$, i.e. \mbox{$r = \frac{1}{|\mathcal{D}(\bs, \ba)|} \sum_{\bs_i, \ba_i in \mathcal{D}} \mathbf{1}_{\bs_i = \bs, \ba_i = \ba} \cdot r(\bs, \ba)$}, and $\hat{\transitions}(\bs'|\bs, \ba)$ is the empirical transition matrix.
Note that expectation under $\policy(\ba|\bs)$ can be computed exactly, since it does not depend on the dataset. The empirical Bellman operator can take higher values as compared to the actual Bellman operator, $\bellman^\policy$, for instance, in an MDP with stochastic dynamics, where $\mathcal{D}$ may not contain transitions to all possible next-states $\bs'$ that can be reached by executing action $\ba$ at state $\bs$, and only contains an optimistic transition. 

We next show how the CQL lower bound result (Theorem~\ref{thm:cql_underestimates}) can be modified to guarantee a lower bound even in this presence of sampling error. To note this, following prior work~\citep{jaksch2010near,osband2017posterior}, we assume concentration properties of the reward function and the transition dynamics:
\begin{assumption}
    $\forall~ \bs, \ba \in \mathcal{D}$, the following relationships hold with high probability, $\geq 1 - \delta$
    \begin{equation*}
        |r - r(\bs, \ba)| \leq \frac{C_{r, \delta}}{\sqrt{|\mathcal{D}(\bs, \ba)|}}, ~~~ ||\hat{\transitions}(\bs'|\bs, \ba) - \transitions(\bs'|\bs, \ba)||_{1} \leq \frac{C_{\transitions, \delta}}{\sqrt{|\mathcal{D}(\bs, \ba)|}}.
    \end{equation*}
\end{assumption}

Under this assumption, the difference between the empirical Bellman operator and the actual Bellman operator can be bounded:
\begin{align*}
    \left\vert\left(\hat{\bellman}^\policy \hat{Q}^k \right) - \left({\bellman}^\policy \hat{Q}^k \right)\right\vert &= \left\vert\left(r - r(\bs, \ba)\right) + \gamma \sum_{\bs'} \left(\hat{\transitions}(\bs'|\bs, \ba) - \transitions(\bs'|\bs,\ba)\right) \E_{\policy(\ba'|\bs')}\left[\hat{Q}^k(\bs' , \ba')\right]\right\vert\\
    &\leq \left\vert r - r(\bs, \ba)\right\vert + \gamma \left\vert \sum_{\bs'} \left(\hat{\transitions}(\bs'|\bs, \ba) - \transitions(\bs'|\bs,\ba)\right) \E_{\policy(\ba'|\bs')}\left[\hat{Q}^k(\bs' , \ba')\right]\right\vert\\
    &\leq \frac{C_{r, \delta} + \gamma C_{T, \delta} 2R_{\max} / (1 - \gamma)}{\sqrt{|\mathcal{D}(\bs, \ba)|}}. 
\end{align*}
This gives us an expression to bound the potential overestimation that can occur due to sampling error, as a function of a constant, $C_{r, \transitions, \delta}$ that can be expressed as a function of $C_{r, \delta}$ and $C_{T, \delta}$, and depends on $\delta$ via a $\sqrt{\log (1/\delta)}$ dependency. This is similar to how prior works have bounded the sampling error due to an empirical Bellman operator~\citep{osband2017posterior,jaksch2010near}.

\subsection{Safe Policy Improvement Guarantee for CQL}
\label{app:safe_pi}
In this section, we prove a safe policy improvement guarantee for CQL (and this analysis is also applicable in general to policy constraint methods with appropriate choices of constraints as we will show). We define the empirical MDP, $\hat{M}$ as the MDP formed by the transitions in the replay buffer, $\hat{M} = \{\bs, \ba, r, \bs' \in \mathcal{D} \}$, and let $J(\pi, \hat{M})$ denote the return of a policy $\pi(\ba|\bs)$ in MDP $\hat{M}$. Our goal is to show that $J(\policy, M) \geq J(\hatbehavior, M) - \varepsilon$, with high probability, where $\varepsilon$ is a small constant. We start by proving that CQL optimizes a penalized RL objective in the empirical MDP, $\hat{M}$.

\begin{lemma}
Let $\hat{Q}^\pi$ be the fixed point of Equation~\ref{eqn:modified_policy_eval}, then $\policy^*(\ba|\bs) := \arg\max_{\policy} \mathbb{E}_{\bs \sim \rho(\bs)}[\hat{V}^\pi(\bs)]$ is equivalently obtained by solving:
\begin{equation}
\label{eqn:policy_optimality}
    \policy^*(\ba|\bs) \leftarrow \arg \max_{\policy}~~ J(\pi, \hat{M}) - \alpha \frac{1}{1 - \gamma} \mathbb{E}_{\bs \sim d^\policy_{\hat{M}}(\bs)}\left[D_{\text{CQL}}(\policy, \hatbehavior)(\bs) \right],
\end{equation}
where $D_{\text{CQL}}(\policy, \hatbehavior)(\bs) := \sum_{\ba} \policy(\ba|\bs) \cdot \left(\frac{\policy(\ba|\bs)}{\hatbehavior(\ba|\bs)} - 1 \right)$.
\end{lemma}
\begin{proof}
$\hat{Q}^\pi$ is obtained by solving a recursive Bellman fixed point equation in the empirical MDP $\hat{M}$, with an altered reward, $r(\bs, \ba) - \alpha \left[ \frac{\policy(\ba|\bs)}{\hatbehavior(\ba|\bs)} - 1 \right]$, hence the optimal policy $\policy^*(\ba|\bs)$ obtained by optimizing the value under the CQL Q-function equivalently is characterized via Equation~\ref{eqn:policy_optimality}.
\end{proof}

Now, our goal is to relate the performance of $\policy^*(\ba|\bs)$ in the \emph{actual} MDP, $M$, to the performance of the behavior policy, $\behavior(\ba|\bs)$. To this end, we prove the following theorem:
\begin{theorem}
\label{thm:zeta_safe_app}
Let $\policy^*(\ba|\bs)$ be the policy obtained by optimizing Equation~\ref{eqn:policy_optimality}. Then the performance of $\policy^*(\ba|\bs)$ in the actual MDP $M$ satisfies,
\begin{multline}
    \label{eqn:performance_relation}
    J(\policy^*, M) \geq J(\hatbehavior, M) - 2\left({\frac{C_{r, \delta}}{1 - \gamma} + \frac{\gamma R_{\max} C_{T, \delta}}{(1 - \gamma)^2}} \right)\mathbb{E}_{\bs \sim d^{\policy^*}_{\hat{M}}(\bs)}\left[ \frac{\sqrt{|\mathcal{A}|}}{\sqrt{|\mathcal{D}(\bs)|}} \sqrt{ D_{\text{CQL}}(\policy^*, \hatbehavior)(\bs) + 1} \right] \\
    + \alpha \frac{1}{1 - \gamma} \mathbb{E}_{\bs \sim d^{\policy^*}_{\hat{M}}(\bs)}\left[D_{\text{CQL}}(\policy^*, \hatbehavior)(\bs) \right].
\end{multline}
\end{theorem}
\begin{proof}
The proof for this statement is divided into two parts. The first part involves relating the return of $\policy^*(\ba|\bs)$ in MDP $\hat{M}$ with the return of $\hatbehavior$ in MDP $\hat{M}$. Since, $\policy^*(\ba|\bs)$ optimizes Equation~\ref{eqn:policy_optimality}, we can relate $J(\policy^*, \hat{M})$ and $J(\hatbehavior, \hat{M})$ as:
\begin{equation*}
    J(\policy^*, \hat{M}) - \alpha \mathbb{E}_{\bs \sim d^\policy_{\hat{M}}(\bs)}\left[D_{\text{CQL}}(\policy^*, \behavior)(\bs) \right] \geq J(\hatbehavior, \hat{M)} - 0 = J(\hatbehavior, \hat{M}).
\end{equation*}
The next step involves using concentration inequalities to upper and lower bound $J(\policy^*, \hat{M})$ and $J(\policy^*, M)$ and the corresponding difference for the behavior policy. In order to do so, we prove the following lemma, that relates $J(\policy, M)$ and $J(\policy, \hat{M})$ for an arbitrary policy $\policy$. We use this lemma to then obtain the proof for the above theorem.
\end{proof}

\begin{lemma}
For any MDP $M$, an empirical MDP $\hat{M}$ generated by sampling actions according to the behavior policy $\hatbehavior(\ba|\bs)$ and a given policy $\policy$,
\begin{equation*}
    \left\vert J(\policy, \hat{M}) - J(\policy, M) \right\vert \leq \left({\frac{C_{r, \delta}}{1 - \gamma} + \frac{\gamma R_{\max} C_{T, \delta}}{(1 - \gamma)^2}} \right)\mathbb{E}_{\bs \sim d^{\policy}_{\hat{M}}(\bs)}\left[ \frac{\sqrt{|\mathcal{A}|}}{{|\mathcal{D}(\bs)|}} \sqrt{ D_{\text{CQL}}(\policy, \hatbehavior)(\bs) + 1} \right].
\end{equation*}
\end{lemma}
\begin{proof}
To prove this, we first use the triangle inequality to clearly separate reward and transition dynamics contributions in the expected return.
\begin{align}
    & \left\vert J(\policy, \hat{M}) - J(\policy, M) \right\vert = \frac{1}{1 - \gamma} \left\vert \sum_{\bs, \ba} d^\policy_{\hat{M}}(\bs) \policy(\ba|\bs) r_{\hat{M}}(\bs, \ba) -  \sum_{\bs, \ba} d^\policy_{{M}}(\bs) \policy(\ba|\bs) r_{{M}}(\bs, \ba) \right\vert\\
    &\leq \frac{1}{1 - \gamma} \left\vert \sum_{\bs, \ba} d^\policy_{\hat{M}}(\bs) \underbrace{\left[ \policy(\ba|\bs) ( r_{\hat{M}}(\bs, \ba) - r_M(\bs, \ba) \right]}_{:= \Delta_1(\bs)} \right\vert + \frac{1}{1 - \gamma} \left\vert \sum_{\bs, \ba} \left(d^\policy_{\hat{M}}(\bs) - d^\policy_{{M}}(\bs) \right) \policy(\ba|\bs) r_M(\bs, \ba) \right\vert
\end{align}
We first use concentration inequalities to upper bound $\Delta_1(\bs)$. Note that under concentration assumptions, and by also using the fact that $\mathbb{E}[\Delta_1(\bs)] = 0$ (in the limit of infinite data), we get:
\begin{align*}
    \vert\Delta_1(\bs)\vert &~ \leq \sum_{\ba} \policy(\ba|\bs) \vert r_{\hat{M}}(\bs, \ba) - r_{M}(\bs, \ba)\vert \leq \sum_{\ba} \policy(\ba|\bs) \frac{C_{r, \delta}}{\sqrt{|\mathcal{D}(\bs)| \cdot |\mathcal{D}(\ba|\bs)|}} = \frac{C_{r, \delta}}{\sqrt{|\mathcal{D}(\bs)|}} \sum_{\ba} \frac{\policy(\ba|\bs)}{\sqrt{\hatbehavior(\ba|\bs)}},
\end{align*}
where the last step follows from the fact that $|\mathcal{D}(\bs, \ba)| = |\mathcal{D}(\bs)| \cdot \hatbehavior(\ba|\bs)$.

Next we bound the second term. We first note that if we can bound $\vert\vert d^\policy_{\hat{M}} - d^\policy_M \vert\vert_1$, (i.e., the total variation between the marginal state distributions in $\hat{M}$ and $M$, then we are done, since $|r_M(\bs, \ba)| \leq R_{\max}$ and $\policy(\ba|\bs) \leq 1$, and hence we bound the second term effectively. We use an analysis similar to \citet{achiam2017constrained} to obtain this total variation bound. Define, $G = (I - \gamma P^{\policy}_{M})^{-1}$ and $\bar{G} = (I - \gamma P^\policy_{\hat{M}})^{-1}$ and let $\Delta = P^\policy_{M} - P^{\policy}_{\hat{M}}$. Then, we can write:
\begin{equation*}
    d^\policy_{\hat{M}} - d^\policy_{M} = (1 - \gamma) (\bar{G} - G) \rho,
\end{equation*}
where $\rho(\bs)$ is the initial state distribution, which is assumed to be the same for both the MDPs. Then, using Equation 21 from \citet{achiam2017constrained}, we can simplify this to obtain,
\begin{equation*}
    d^\policy_{{M}} - d^\policy_{\hat{M}} = (1 - \gamma) \gamma {G} \Delta \bar{G} \mu = \gamma \bar{G} \Delta d^\policy_{\hat{M}}
\end{equation*}
Now, following steps similar to proof of Lemma 3 in \citet{achiam2017constrained} we obtain,
\begin{align*}
    \vert\vert \Delta d^\policy_M||_{1} &= \sum_{\bs'} \left\vert \sum_{\bs} \Delta(\bs'|\bs) d^\policy_{M}(\bs) \right\vert \leq \sum_{\bs, \bs'} \left\vert \Delta(\bs'|\bs) \right\vert d^\policy_{\hat{M}}\\
    &= \sum_{\bs, \bs'} \left\vert \sum_{\ba} \left(P_{\hat{M}}(\bs'|\bs, \ba) - P_{M}(\bs'|\bs, \ba) \right) \policy(\ba|\bs) \right\vert d^\policy_{\hat{M}}(\bs) \\
    &\leq \sum_{\bs, \ba} \vert\vert P_{\hat{M}}(\cdot|\bs, \ba) - P_M(\cdot|\bs, \ba) \vert\vert_{1} \policy(\ba|\bs) d^\policy_{\hat{M}}(\bs)\\
    &\leq \sum_{\bs} d^\policy_{\hat{M}}(\bs) \frac{C_{T, \delta}}{\sqrt{|\mathcal{D}(\bs)|}} \sum_{\ba}  \frac{\policy(\ba|\bs)}{\sqrt{\hatbehavior(\ba|\bs)}}.
\end{align*}
Hence, we can bound the second term by:
\begin{align*}
    \left\vert \sum_{\bs} \left(d^\policy_{\hat{M}}(\bs) - d^\policy_{{M}}(\bs) \right) \policy(\ba|\bs) r_M(\bs, \ba) \right\vert \leq \frac{\gamma C_{T, \delta} R_{\max}}{(1 - \gamma)} \sum_{\bs} d^\policy_{\hat{M}}(\bs) \frac{1}{\sqrt{|\mathcal{D}(\bs)|}} \sum_{\ba}  \frac{\policy(\ba|\bs)}{\sqrt{\hatbehavior(\ba|\bs)}}.
\end{align*}

To finally obtain $D_{\text{CQL}}(\policy, \behavior)(\bs)$ in the bound, let $\alpha(\bs, \ba) := \frac{\policy(\ba|\bs)}{\sqrt{\hatbehavior(\ba|\bs)}}$. Then, we can write $D_{\text{CQL}}(\policy, \hatbehavior)(\bs)$ as follows:
\begin{align*}
    D_\text{CQL}(\bs) &= \sum_{\ba} \frac{\policy(\ba|\bs)^2}{\hatbehavior(\ba|\bs)} - 1\\
    \implies D_{\text{CQL}}(\bs) + 1 &= \sum_{\ba} \alpha(\bs, \ba)^{2}\\
    \implies D_{\text{CQL}}(\bs) + 1 &\leq \left(\sum_{\ba} \alpha(\bs,\ba)\right)^2 \leq |\mathcal{A}| \left(D_\text{CQL}(\bs) + 1 \right).
\end{align*}
Combining these arguments together, we obtain the following upper bound on $|J(\policy, M) - J(\policy, \hat{M})|$,
\begin{equation*}
    \left\vert J(\policy, \hat{M}) - J(\policy, M) \right\vert \leq \left({\frac{C_{r, \delta}}{1 - \gamma} + \frac{\gamma R_{\max} C_{T, \delta}}{(1 - \gamma)^2}} \right)\mathbb{E}_{\bs \sim d^{\policy}_{\hat{M}}(\bs)}\left[ \frac{\sqrt{|\mathcal{A}|}}{\sqrt{|\mathcal{D}(\bs)|}} \sqrt{ D_{\text{CQL}}(\policy, \hatbehavior)(\bs) + 1} \right].
\end{equation*}
\end{proof}
The proof of Theorem~\ref{thm:zeta_safe_app} is then completed by using the above Lemma for bounding the sampling error for $\policy^*$ and then upper bounding the sampling error for $\hatbehavior$ by the corresponding sampling error for $\policy^*$, hence giving us a factor of $2$ on the sampling error term in Theorem~\ref{thm:zeta_safe_app}. To see why this is mathematically correct, note that $D_{\text{CQL}}(\hatbehavior, \hatbehavior)(\bs) = 0$, hence $\sqrt{D_{\text{CQL}}(\policy^*, \hatbehavior)(\bs) + 1} \geq \sqrt{D_{\text{CQL}}(\hatbehavior, \hatbehavior)(\bs) + 1}$, which means the sampling error term for $\policy^*$ pointwise upper bounds the sampling error for $\hatbehavior$, which justifies the factor of $2$.

\section{Extended Related Work and Connections to Prior Methods}
\label{sec:extended_related_work}

In this section, we discuss related works to supplement Section~\ref{sec:related}. Specifically, we discuss the relationships between CQL and uncertainty estimation and policy-constraint methods.

\textbf{Relationship to uncertainty estimation in offline RL.} 
A number of prior approaches to offline RL estimate some sort of epistemic uncertainty to determine the trustworthiness of a Q-value prediction~\citep{kumar2019stabilizing,fujimoto2018off,agarwal2019optimistic,levine2020offline}.
The policy is then optimized with respect to lower-confidence estimates derived using the uncertainty metric. However, it has been empirically noted that uncertainty-based methods are not sufficient to prevent against OOD actions~\citep{fujimoto2018off,kumar2019stabilizing} in and of themselves, are often augmented with policy constraints due to the inability to estimate tight and calibrated uncertainty sets. Such loose or uncalibrated uncertainty sets are still effective in providing exploratory behavior in standard, online RL~\citep{osband2016deep,osband2017posterior}, where these methods were originally developed. However, offline RL places high demands on the fidelity of such sets~\citep{levine2020offline}, making it hard to directly use these methods.

\textbf{How does CQL relate to prior uncertainty estimation methods?} Typical uncertainty estimation methods rely on learning a pointwise upper bound on the Q-function that depends on epistemic uncertainty~\citep{jaksch2010near,osband2017posterior} and these upper-confidence bound values are then used for exploration. In the context of offline RL, this means learning a pointwise lower-bound on the Q-function. We show in Section~\ref{sec:policy_eval} that, with a na\"ive choice of regularizer (Equation~\ref{eqn:objective_1}), we can learn a uniform lower-bound on the Q-function, however, we then showed that we can improve this bound since the value of the policy is the primary quantity of interest that needs to be lower-bounded. This implies that CQL strengthens the popular practice of point-wise lower-bounds made by uncertainty estimation methods.

\textbf{Can we make CQL dependent on uncertainty?} We can slightly modify CQL to make it be account for epistemic uncertainty under certain statistical concentration assumptions. Typical uncertainty estimation methods in RL~\citep{osband2016deep,jaksch2010near} assume the applicability of concentration inequalities (for example, by making sub-Gaussian assumptions on the reward and dynamics), to obtain upper or lower-confidence bounds and the canonical amount of over- (under-) estimation is usually given by,  $\mathcal{O}\left(\frac{1}{\sqrt{n(\bs, \ba)}}\right)$, where $n(\bs, \ba)$ is the number of times a state-action pair $(\bs, \ba)$ is observed in the dataset. We can incorporate such behavior in CQL by modifying Equation~\ref{eqn:cql_framework} to update Bellman error weighted by the cardinality of the dataset, $|\mathcal{D}|$, which gives rise to the following effective Q-function update in the tabular setting, without function approximation:
\begin{equation*}
    \hat{Q}^{k+1}(\bs, \ba) = \bellman^\pi \hat{Q}^k(\bs, \ba) - \alpha \frac{\mu(\ba|\bs) - \behavior(\ba|\bs)}{n(\bs,\ba)} \rightarrow \bellman^\pi \hat{Q}^k(\bs, \ba) \text{~~as~~} n(\bs, \ba) \rightarrow \infty.
\end{equation*}
In the limit of infinite data, i.e. $n(\bs, \ba) \rightarrow \infty$, we find that the amount of underestimation tends to $0$. When only a finite-sized dataset is provided, i.e. $n(\bs, \ba) < N$, for some $N$, we observe that by making certain assumptions, previously used in prior work~\citep{jaksch2010near,osband2016deep} on the concentration properties of the reward value, $r(\bs, \ba)$ and the dynamics function, $\transitions(\bs'|\bs, \ba)$, such as follows:
\begin{equation*}
    ||\hat{r}(\bs, \ba) - r(\bs, \ba)|| \leq \frac{C_r}{\sqrt{n(\bs, \ba)}} ~~~ \text{and} ~~~ ||\hat{\transitions}(\bs'|\bs, \ba) - \transitions(\bs'|\bs, \ba)|| \leq \frac{C_\transitions}{\sqrt{n(\bs, \ba)}}, 
\end{equation*}
where $C_r$ and $C_\transitions$ are constants, that depend on the concentration properties of the MDP, and by appropriately choosing $\alpha$, i.e. $\alpha = \Omega (n(\bs, \ba))$, such that the learned Q-function still lower-bounds the actual Q-function (by nullifying the possible overestimation that appears due to finite samples), we are still guaranteed a lower bound.

\section{Additional Experimental Setup and Implementation Details}
\label{sec:experimental_details}
In this section, we discuss some additional implementation details related to our method. As discussed in Section~\ref{sec:practical_alg}, CQL can be implemented as either a Q-learning or an actor-critic method. For our experiments on D4RL benchmarks~\citep{d4rl}, we implemented CQL on top of soft actor-critic (SAC)~\citep{haarnoja}, and for experiments on discrete-action Atari tasks, we implemented CQL on top of QR-DQN~\citep{dabney2018distributional}.
We experimented with two ways of implementing CQL, first with a fixed $\alpha$, where we chose $\alpha=5.0$, and second with a varying $\alpha$ chosen via dual gradient-descent. The latter formulation automates the choice of $\alpha$ by introducing a ``budget'' parameter, $\tau$, as shown below:
\begin{equation}
    \small{\min_{Q} \textcolor{red}{\max_{\alpha \geq 0}}~ \alpha \left( \E_{\bs \sim d^\behavior(\bs)}\left[\log \sum_{\ba} \exp(Q(\bs, \ba))\!-\!\E_{\ba \sim \behavior(\ba|\bs)}\left[Q(\bs, \ba)\right]\right] - \textcolor{red}{\tau}\right) \!+\!\frac{1}{2}\!\E_{\bs, \ba, \bs' \sim \mathcal{D}}\left[\left(Q - \bellman^{\policy_k} \hat{Q}^{k} \right)^2 \right]\!.}
    \label{eqn:practical_objective_with_lagrange}
\end{equation}
Equation~\ref{eqn:practical_objective_with_lagrange} implies that if the expected difference in Q-values is less than the specified threshold $\tau$, $\alpha$ will adjust to be close to $0$, whereas if the difference in Q-values is higher than the specified threshold, $\tau$, then $\alpha$ is likely to take on high values, and thus more aggressively penalize Q-values. We refer to this version as CQL-Lagrange, and we found that this version outperforms the version with a fixed $\alpha$ on the gym-MuJoCo D4RL benchmarks, and drastically outperforms the fixed $\alpha$ version on the more complex AntMazes. 

\textbf{Choice of $\alpha$.} Our experiments in Section~\ref{sec:experiments} use the Lagrange version to automatically tune $\alpha$ during training. 
For our experiments, across D4RL Gym MuJoCo domains, we choose $\tau = 10.0$. For the other D4RL domains (Franka Kitchen and Adroit), we chose $\tau = 5.0$. However, for our Atari experiments, we used a fixed penalty, with $\alpha=1.0$ chosen uniformly for Table~\ref{table:atari_reduced_size} (with 10\% data), $\alpha=4.0$ chosen uniformly for Table~\ref{table:atari_reduced_size} (with 1\% data), and $\alpha=0.5$ for Figure~\ref{fig:cql_20m_atari}.

\textbf{Computing $\log \sum_{\ba} \exp(Q(\bs, \ba)$.} CQL($\mathcal{H}$) uses log-sum-exp in the objective for training the Q-function (Equation~\ref{eqn:practical_objective}). In discrete action domains, we compute the log-sum-exp exactly by invoking the standard \texttt{tf.reduce\_logsumexp()} (or \texttt{torch.logsumexp()}) functions provided by autodiff libraries. In continuous action tasks, CQL($\mathcal{H}$) uses importance sampling to compute this quantity, where in practice, we sampled \textbf{10} action samples each at every state $\bs$ from a uniform-at-random $\text{Unif}(\ba)$ and the current policy, $\policy(\ba|\bs)$ and used these alongside importance sampling to compute it as follows using $N = 10$ action samples:
\begin{align*}
    \log \sum_{\ba}\exp (Q(\bs, \ba)) &= \log \left( \frac{1}{2}\sum_{\ba} \exp (Q(\bs, \ba)) + \frac{1}{2} \sum_{\ba} \exp (Q(\bs, \ba)) \right) \\
    &= \log \left( \frac{1}{2} \E_{\ba \sim \text{Unif}(\ba)}\left[ \frac{\exp (Q(\bs, \ba))}{\text{Unif}(\ba)}\right] + \frac{1}{2} \E_{\ba \sim \policy(\ba|\bs)}\left[ \frac{\exp (Q(\bs, \ba))}{\policy(\ba|\bs)}\right] \right) \\
    &\approx \log \left( \frac{1}{2N} \sum_{\ba_i \sim \text{Unif}(\ba)}^N \left[ \frac{\exp (Q(\bs, \ba_i))}{\text{Unif}(\ba)}\right] + \frac{1}{2N} \sum_{\ba_i \sim \policy(\ba|\bs)}^N \left[ \frac{\exp (Q(\bs, \ba_i))}{\policy(\ba_i|\bs)}\right] \right).
\end{align*}

\textbf{Hyperparameters.} For the D4RL tasks, we built CQL on top of the implementation of SAC provided at: \url{https://github.com/vitchyr/rlkit/}. Our implementation mimicked the RLkit SAC algorithm implementation, with the exception of a smaller policy learning rate, which was chosen to be 3e-5 or 1e-4 for continuous control tasks. Following the convention set by D4RL~\citep{d4rl}, we report the normalized, smooth average undiscounted return over 4 seeds for in our results in Section~\ref{sec:experiments}.

The other hyperparameters we evaluated on during our preliminary experiments, and might be helpful guidelines for using CQL are as follows:
\begin{itemize}
    \item \textbf{Q-function learning rate.} We tried two learning rate values $[1e-4, 3e-4]$ for the Q-function. We didn't observe a significant difference in performance across these values. $3e-4$ which is the SAC default was chosen to be the default for CQL.
    \item \textbf{Policy learning rate.} We evaluated CQL with a policy learning rate in the range of $[3e-5, 1e-4, 3e-4$. We found $3e-5$ to almost uniformly attain good performance. While $1e-4$ seemed to be better on some of our experiments (such as hopper-medium-v0 and antmaze-medium-play-v0), but it performed badly with the real-human demonstration datasets, such as the Adroit tasks. We chose $3e-5$ as the default across all environments.
    \item \textbf{Lagrange threshold $\tau$.} We ran our preliminary experiments with three values of threshold, $\tau = [2.0, 5.0, 10.0]$. However, we found that $\tau=2.0$, led to a huge increase in the value of $\alpha$ (sometimes upto the order of millions), and as a result, highly underestimated Q-functions on all domains (sometimes upto the order of -1e6). On the datasets with human demonstrations -- the Franka Kitchen and Adroit domains, we found that $\tau=5.0$ obtained lower-bounds on Q-values, whereas $\tau=10.0$ was unable to prevent overestimation in Q-values in a number of cases and Q-values diverged to highly positive values (> 1e+6). For the MuJoCo domains, we observed that $\tau=10.0$ or $\tau=5.0$ gave rise to a stable curve of Q-values, and hence this threshold was chosen for these experiments. Note that none of these hyperparameter selections required any notion of onine evaluation, since these choices were made based on the predicted Q-values on dataset state-actions pairs.  
    \item \textbf{Number of gradient steps.} We evaluated our method on varying number of gradient steps. Since CQL uses a reduced policy learning rate (3e-5), we trained CQL methods for 1M gradient steps. Due to a lack of a proper valdiation error metric for offline Q-learning methods, deciding the number of gradient steps dynamically has been an open problem in offline RL.~\citep{levine2020offline}. Different prior methods choose the number of steps differently. For our experiments, we used 1M gradient steps for all D4RL domains, and followed the convention from \citet{agarwal2019optimistic}, to report returns after 5X gradient steps of training for the Atari results in Table~\ref{table:atari_reduced_size}.  
    \item \textbf{Choice of Backup.} Instead of using an actor-critic version of CQL, we can instead also use an approximate max-backup for the Q-function in a continuous control setting. Similar to prior work~\citep{kumar2019stabilizing}, we sample 10 actions from the current policy at the next state $\bs'$, called $\ba_1, \cdots \ba_{10} \sim \policy(\ba'|\bs')$ and generate the target values for the backup using the following equation: $r(\bs, \ba) + \max_{\ba_1, \cdots, \ba_{10}} Q(\bs', \ba')$. This is different from the standard actor-critic backup that performs an expectation of the Q-values $Q(\bs', \ba')$ at the next state under the policy's distribution $\policy(\ba'|\bs')$. In certain environments, such as in the Franka Kitchen and AntMaze domains using these backups performs better than the reported numbers with actor-critic algorithm.
\end{itemize}

Other hyperparameters, were kept identical to SAC on the D4RL tasks, including the twin Q-function trick, soft-target updates, etc.
In the Atari domain, we based our implementation of CQL on top of the QR-DQN implementation provided by \citet{agarwal2019optimistic}. We did not tune any parameter from the QR-DQN implementation released with the official codebase with \citep{agarwal2019optimistic}.

\section{Ablation Studies}
\label{app:additional_results}
In this section. we describe the experimental findings of some ablations for CQL. Specifically we aim to answer the following questions: 
\begin{enumerate}
    \item How does CQL($\mathcal{H}$) compare to CQL($\rho$), with $\rho = \hat{\policy}^{k-1}$, the previous policy?
    \item How does the CQL variant which uses Equation~\ref{eqn:objective_1} compare to CQL($\mathcal{H}$) variant in Equation~\ref{eqn:practical_objective}, that results in a tighter lower-bound theoretically?
    \item How do the Lagrange (Equation~\ref{eqn:practical_objective_with_lagrange}) and non-Lagrange (Equation~\ref{eqn:practical_objective}) formulations of CQL($\mathcal{H}$) empirically compare to each other?
\end{enumerate}

We start with question \textbf{(1)}. On three MuJoCo tasks from D4RL, we evaluate the performance of both CQL($\mathcal{H}$) and CQL($\rho$), as shown in Table~\ref{table:mujoco_cql_ablation}. We observe that on these tasks, CQL($\mathcal{H}$) generally performs better than CQL($\rho$). However, when a sampled estimate of log-sum-exp of the Q-values becomes inaccurate due to high variance importance weights, especially in large action spaces, such as in the Adroit environments, we observe in Table~\ref{table:adroit_antmaze} that CQL($\rho$) tends to perform better.

\begin{table}[h]
    \centering
    \begin{tabular}{l|r|r}
    \hline
        \textbf{Task Name} & \textbf{CQL($\mathcal{H})$} & \textbf{CQL($\rho$)} \\
        \hline
        halfcheetah-medium-expert & \textbf{7234.5} & 3995.6 \\
        walker2d-mixed & \textbf{1227.2} & 812.7  \\
        hopper-medium & \textbf{1866.1} & 1166.1 \\
        \hline
    \end{tabular}
    \caption{{\small Average return obtained by CQL($\mathcal{H}$), and CQL($\rho$) on three D4RL MuJoCo environments. Observe that on these environments, CQL($\mathcal{H}$) generally outperforms CQL($\rho$).}}
    \label{table:mujoco_cql_ablation}
    \vspace{-15pt}
\end{table}

Next, we evaluate the answer to question \textbf{(2)}. On three MuJoCo tasks from D4RL, as shown in Table~\ref{table:data_subtract}, we evaluate the performance of CQL($\mathcal{H}$), with and without the dataset maximization term in Equation~\ref{eqn:modified_policy_eval}. We observe that omitting this term generally seems to decrease performance, especially in cases when the dataset distribution is generated from a single policy, for example, hopper-medium.

\begin{table}[h]
    \centering
    \begin{tabular}{l|r|r}
    \hline
        \textbf{Task Name} & \textbf{CQL($\mathcal{H})$} & \textbf{CQL($\mathcal{H}$) (w/ Equation~\ref{eqn:objective_1})} \\
        \hline
        hopper-medium-expert & 3628.4 & 3610.3 \\
        hopper-mixed & \textbf{1563.2} & 864.6 \\
        hopper-medium & \textbf{1866.1} & 1028.4 \\
        \hline
    \end{tabular}
    \caption{{\small Average return obtained by CQL($\mathcal{H}$) and CQL($\mathcal{H}$) without the dataset average Q-value maximization term. The latter formulation corresponds to Equation~\ref{eqn:objective_1}, which is void of the dataset Q-value maximization term. We show in Theorem~\ref{thm:cql_underestimates} that Equation~\ref{eqn:objective_1} results in a weaker lower-bound. In this experiment, we also observe that this approach is generally outperformed by CQL($\mathcal{H}$).}}
    \vspace{-15pt}
    \label{table:data_subtract}
\end{table}

Next, we answer question \textbf{(3)}. Table~\ref{table:lagrangian} shows the performance of CQL and CQL-Lagrange on some gym-MuJoCo and AntMaze tasks from D4RL. In all of these tasks, we observe that the Lagrange version (Equation~\ref{eqn:practical_objective_with_lagrange}, which automates the choice of $\alpha$ using dual gradient descent on the CQL regularizer, performs better than the non-Lagrange version (Equation~\ref{eqn:practical_objective}). In some cases, for example, the AntMazes, the difference can be as high as 30\% of the maximum achievable performance. On the gym MuJoCo tasks, we did not observe a large benefit of using the Lagrange version, however, there are some clear benefits, for instance in the setting when learning from hopper-mixed dataset.

\begin{table}[h]
    \centering
    \begin{tabular}{l|r|r}
    \hline
        \textbf{Task Name} & \textbf{CQL($\mathcal{H})$ (Lagrange), $\tau = 10.0$} & \textbf{CQL($\mathcal{H}$), $\alpha = 5.0$}  \\
        \hline
        hopper-medium-expert & 3628.4 & 3589.4\\
        hopper-mixed & {\textbf{1866.1}} & 1002.8 \\
        walker-mixed & {1227.2} & 1055.7 \\
        walker2d-random-expert & 4183.0 & 3934.5 \\ 
        hopper-medium & 1866.1 & 1583.4 \\ \hline \hline
        antmaze-medium-diverse & \textbf{0.53} & 0.21 \\
        antmaze-large-play & \textbf{0.15} & 0.02 \\
        antmaze-large-diverse & \textbf{0.14}  & 0.05 \\
        \hline
    \end{tabular}
    \caption{{\small Average return obtained by CQL($\mathcal{H}$) and CQL($\mathcal{H}$) with automatic tuning for $\alpha$ by using a Lagrange version. Observe that both versions are generally comparable, except in the AntMaze tasks, where an adaptive value of $\alpha$ greatly outperforms a single chosen value of $\alpha$.}}
    \label{table:lagrangian}
    \vspace{-15pt}
\end{table}

\end{document}